\documentclass{article} 
\usepackage[T1]{fontenc}
\usepackage{iclr2027_conference,times}
\renewcommand{\headrulewidth}{0pt}

\usepackage{amsmath,amsfonts,bm}

\def\1{\bm{1}}

\DeclareMathAlphabet{\mathsfit}{\encodingdefault}{\sfdefault}{m}{sl}
\SetMathAlphabet{\mathsfit}{bold}{\encodingdefault}{\sfdefault}{bx}{n}

\newcommand{\E}{\mathbb{E}}

\usepackage{amsmath,amssymb,amsfonts,amsthm,mathtools,bm}
\usepackage{algorithm}
\usepackage{microtype}
\usepackage{hyperref}
\usepackage{url}
\usepackage{algpseudocode}
\hypersetup{hidelinks}
\usepackage{enumitem}

\usepackage{booktabs,makecell,graphicx}

\newtheoremstyle{boldplain}
  {\topsep}{\topsep}{\itshape}{}{\normalfont}{.}{.5em}
  {\thmname{\textbf{#1}}\thmnumber{\textbf{ #2}}\thmnote{ (#3)}}
\newtheoremstyle{bolddefinition}
  {\topsep}{\topsep}{\normalfont}{}{\normalfont}{.}{.5em}
  {\thmname{\textbf{#1}}\thmnumber{\textbf{ #2}}\thmnote{ (#3)}}

\theoremstyle{boldplain}
\newtheorem{theorem}{Theorem}
\newtheorem{lemma}{Lemma}
\newtheorem{proposition}{Proposition}
\newtheorem{corollary}{Corollary}
\theoremstyle{bolddefinition}
\newtheorem{assumption}{Assumption}
\newtheorem{definition}{Definition}
\theoremstyle{bolddefinition}

\newtheoremstyle{restatedstyle}
  {3pt}{3pt}{\itshape}{}{\normalfont}{.}{.5em}
  {\thmname{\textbf{#1}}\thmnumber{\textbf{ #2}}\thmnote{ (#3)}}

\theoremstyle{restatedstyle}
\newtheorem*{lemnsrestated}{Lemma~\ref{lem:ns}}
\theoremstyle{plain}
\theoremstyle{restatedstyle}
\newtheorem*{lemtrackingrestated}{Lemma~\ref{lem:tracking}}
\theoremstyle{boldplain}

\usepackage{amsthm}
\newtheorem*{repproposition}{Proposition~\ref{prop:two-level}}

\newcommand{\normF}[1]{\left\lVert #1\right\rVert_F}

 \newcommand{\nf}[1]{\normF{#1}}

\hypersetup{
    colorlinks=false,
    pdflinkmargin=0.3pt,
    pdfborder={0 0 0.2},
    linkbordercolor={1 0 0},
    citebordercolor={0 0.6 0},
    urlbordercolor={0 0 1}
}

\makeatletter
\let\Muon@natlinkstart\hyper@natlinkstart
\let\Muon@natlinkend\hyper@natlinkend
\let\Muon@natlinkbreak\hyper@natlinkbreak
\newcommand{\Muon@citestrut}{%
  {\XeTeXLinkMargin=0pt
   \XeTeXLinkBox{\vrule width0pt height.7em depth.22em}}}
\renewcommand{\hyper@natlinkstart}[1]{\Muon@natlinkstart{#1}\Muon@citestrut}
\renewcommand{\hyper@natlinkend}{\Muon@citestrut\Muon@natlinkend}
\renewcommand{\hyper@natlinkbreak}[2]{%
  \Muon@citestrut\Muon@natlinkbreak{#1}{#2}\Muon@citestrut}
\makeatother

\title{Convergence of Practical Muon}

\author{
Haonan Wang$^{1}$,~Yu Wu$^{1}$,~Minghui Liwang$^{1,2}$,~
Xinlei Yi$^{1,2}$\thanks{Corresponding author:
\texttt{xinleiyi@tongji.edu.cn}.},~
Yiguang Hong$^{1,2}$\\
$^{1}$Tongji University\\
$^{2}$Shanghai Research Institute for Intelligent Autonomous Systems
}

\iclrfinalcopy 
\begin{document}

\maketitle

\begin{abstract}
Muon is emerging as a promising alternative to AdamW for large-scale neural network training, 
yet theoretical understanding of its practical implementation remains incomplete, 
as existing analyses often simplify or omit two key components: 
(i) practical Newton--Schulz iterations with empirically tuned polynomial coefficients $(3.4445,-4.7750,2.0315)$; 
and (ii) decoupled weight decay for regularization.
In this paper, we provide an optimization interpretation and establish convergence for practical Muon, jointly accounting for both components.
Specifically, we interpret practical Muon as right-preconditioned optimization of the original loss 
with a dynamic weighted $\ell_2$ regularizer that vanishes as stationarity is approached,
so that the optimization target remains the original objective.
We then establish, to our best knowledge, 
the first convergence guarantee for practical Muon in the stochastic nonconvex setting,
with an $\mathcal{O}(T^{-1/4})$ convergence rate in terms of the expected Frobenius norm of the gradient, 
improving the dimension dependence of the best known AdamW's convergence rate by a factor of $\sqrt{d}$, 
where $T$ is the iteration horizon and $d$ is the parameter dimension.
Experiments further support the theoretical convergence results.

\end{abstract}

\section{Introduction}

By exploiting the matrix structure of model parameters,
Muon \citep{jordan2024muon} has been shown the potential to become a dominant optimizer for large-scale neural network training,
displacing adaptive methods such as AdaGrad \citep{duchi2011adaptive}, RMSProp \citep{tieleman2012lecture}, Adam \citep{kingma2015adam}, and AdamW \citep{loshchilov2019decoupled},
which have dominated the field over the past decade,
as well as classical SGD with momentum \citep{robbins1951stochastic,sutskever2013importance}.
Remarkable empirical success has been reported for Muon and its
variants in language modeling 
\citep{liu2025muon,shah2025practical,kimi2025k2,xu2026deepseek},
multimodal and agentic systems \citep{kimi2026k25,zeng2026glm,hong2026glm},
and embodied intelligence \citep{fan2026rethinking}.
In particular, \citet{liu2025muon} report roughly twice AdamW's
computational efficiency on 399M--1.5B dense models,
while DeepSeek-V4-Pro demonstrates Muon's scalability to
1.6 trillion parameters and 33 trillion training tokens
\citep{xu2026deepseek}.

Despite the empirical success, convergence remains an open problem
for practical Muon, which combines two key ingredients:
(i) practical Newton--Schulz (NS) iterations that use  matrix polynomials
with empirically tuned coefficients $(3.4445,-4.7750,2.0315)$
to produce approximately orthogonalized momentum matrices as update directions;
 and
(ii) decoupled weight decay (WD) to provide regularization
 and support stable training at scale \citep{liu2025muon}.
The convergence analysis faces two main challenges.
First, practical NS is difficult to analyze directly, and existing studies often replace it with exact SVD-based orthogonalization \citep{shen2026on},
or standard NS using different polynomial coefficients \citep{kim2026muon,choudhury2026muon} 
which can yield poorer orthogonalization of ill-conditioned momentum matrices than practical NS \citep{jordan2024muon}.
To our best knowledge, neither alternative has been publicly validated in large-scale commercial model training \citep{liu2025muon,kimi2025k2,xu2026deepseek,kimi2026k25,zeng2026glm,hong2026glm}.
As for practical NS without WD,
\citet{do2026muon} assume strong convexity and establish convergence to a neighborhood of the minimizer,
while \citet{qian2026convergence} establish an $\mathcal{O}(\sqrt{r}/T^{1/4})$ convergence rate for
nuclear norm stationarity in the stochastic nonconvex setting,
where $r$ is the smaller dimension of the matrix-valued parameter and $T$ is the iteration horizon.

\begin{table}[t]
\centering
\caption{Comparison of existing convergence results for Muon.
The rate column uses $\mathbb{E}\|\nabla f\|_{\mathrm F}$ unless
a KKT stationarity measure or expected objective gap  (obj.) is indicated.
$\dagger$: The nuclear norm rates established in the original papers
are converted to Frobenius norm rates under our assumptions.
Std. and Pract. denote standard and practical NS, respectively.
}
\label{tab:muon-convergence}
\normalsize
\setlength{\tabcolsep}{3pt}
\renewcommand{\arraystretch}{0.95}
\setlength{\aboverulesep}{2pt}
\setlength{\belowrulesep}{2pt}
\begin{tabular*}{\linewidth}{@{\hspace{4pt}\extracolsep{\fill}}lcccc@{\hspace{4pt}}}
\toprule
Reference & WD & NS & Nonconvex & Convergence rate \\
\midrule
\cite{shen2026on}
& $\times$ & Exact & $\checkmark$
& $\mathcal{O}(\sqrt r/T^{1/4})$ \\
\cite{kim2026muon}
& $\times$ & Standard & $\checkmark$
& $\mathcal{O}(\sqrt r/T^{1/4})^{\dagger}$ \\
\cite{choudhury2026muon}
& $\times$ & Standard & $\checkmark$
& $\mathcal{O}(r/T^{1/4})$ \\
\cite{do2026muon}
& $\times$ & Practical & Strongly convex
& neighborhood (obj.) \\
\cite{qian2026convergence}
& $\times$ & Std./Pract. & $\checkmark$
& $\mathcal{O}(r/T^{1/4})^{\dagger}$ \\
\midrule
\cite{chen2026muon}
& $\checkmark$ & Exact & $\checkmark$
& neighborhood (KKT) \\
\cite{pethick2025training}
& $\checkmark$ & Exact & $\checkmark$
& $\mathcal{O}(r/T^{1/4})$ (KKT) \\
\cite{qian2026convergence}
& $\checkmark$ & Std./Pract. & Star-convex
& $\mathcal{O}(\sqrt{r} (\log{T}/T)^{1/3})$ (obj.) \\
\midrule
\textbf{Ours}
& $\checkmark/\times$ & Std./Pract. & $\checkmark$
& $\boldsymbol{\mathcal{O}(1/T^{1/4})}$ \\
\bottomrule
\end{tabular*}%
\end{table}

Second, WD further complicates the convergence analysis through its interaction with NS.
Existing studies that account for WD typically replace NS with exact 
SVD-based 
orthogonalization,
and show that Muon is equivalent to solving the following constrained optimization problem 
\citep{pethick2025training,chen2026muon}:
\begin{equation}
\min\nolimits_{W\in\mathbb{R}^{m\times n}} f(W)
\quad \text{s.t.} \quad
\|W\|_{\mathrm{op}}\le 1/\lambda, \nonumber
\end{equation}
where $f$ denotes the original loss function,  
$W$ is the matrix-valued parameter, $\lambda>0$ is the WD coefficient, and $\|\cdot\|_{\mathrm{op}}$ denotes the spectral norm.
For stochastic nonconvex objectives, \citet{pethick2025training} establish an $\mathcal{O}(r/T^{1/4})$ convergence rate in terms of a KKT stationarity measure for the above constrained problem.
When both NS and WD are considered, \citet{qian2026convergence} provide the only convergence guarantee that we are aware of.
Under star-convexity and with parameter choices that depend on the norm of an unknown global minimizer, 
they prove convergence to an optimal solution of the original problem.

These observations naturally raise two questions about Muon with practical NS and WD in the stochastic nonconvex setting:

\emph{
\hspace{-0.25em}Q1) What does Muon actually optimize 
 when practical NS and WD are combined?
\\
Q2) Does it converge? If so, at what rate, and can
the convergence rate be sharper than AdamW's?}

\paragraph{Controbutions.}
We address these questions and make the following contributions:
\begin{itemize}[leftmargin=*, labelsep=0.5em]
    \item
    \looseness=-1 
    We interpret Muon with NS and WD as right-preconditioned
    optimization of $f$ with a dynamic weighted $\ell_{2}$ regularizer
    that vanishes as $\|\nabla f\|_{F}\to0$, indicating that
    Muon still targets the original objective.
    While this interpretation is exact when momentum in Muon and noise in stochastic grad\-ient
     are omitted, our convergence guarantees remain valid in their presence.
    Thus we address Q1.
 
    \item 
    To the best of our knowledge, we establish the first convergence
    guarantee for Muon with practical NS and WD in the stochastic
    nonconvex setting, as summarized in Table~\ref{tab:muon-convergence}, achieving
    \vspace{-0.25em}
    $$
    \frac1T\sum\nolimits_{t=0}^{T-1}\E\nf{\nabla f(W_t)}
    =
    \mathcal{O}\Big(\frac{\sqrt[4]{L(f(W_0)-f_{\inf})\sigma^2/b_{\mathrm{mb}}}}{{T^{1/4}}}
    \Big).
    $$
    Notably, the rate above has sharper dimension dependence than
    the best known AdamW's rate \citep{li2026frac} by a factor
    of $\sqrt{d}=\sqrt{mn}$ in the Frobenius norm,
    with AdamW satisfying
    \vspace{-0.25em}
    \begin{align} \label{rate:adamW}
         \frac1T\sum\nolimits_{t=0}^{T-1}\E\|\nabla f(W_t)\|_1 = \mathcal{O}\Big( \frac{\sqrt{d}\sqrt[4]{L(f(W_0)-f_{\inf})\sigma^2/b_{\mathrm{mb}}}}{T^{1/4}} \Big),
    \end{align}
    where
    $\|A\|_1=\sum_{i,j}|A_{ij}|$ denotes
    the entrywise $\ell_1$ norm, and
    the Frobenius norm comparison follows from $\|A\|_F\le\|A\|_1$.
    Our rate also matches that of SGD in the Frobenius norm \citep{ghadimi2013stochastic}
    and that of AdamW in the entrywise $\ell_1$ norm.
    Thus we address Q2.
 \end{itemize}

 \paragraph{Outline.}
 The rest of the paper is organized as follows.
Section~\ref{sec:pre} presents the preliminaries.
Then Sections~\ref{sec:interpret}--\ref{sec:proof}
provide an optimization interpretation of practical Muon, its convergence
results, and the corresponding proofs, respectively.
Section~\ref{sec:advPNS} illustrates practical NS's superiority over standard NS.
Section~\ref{sec:simu} empirically supports Muon's theoretical advantage.
Finally, Section~\ref{sec:con} concludes the paper, with additional related work
in Appendix~\ref{app:realted} and further details in the remaining appendices.

\section{Preliminaries} \label{sec:pre}

We consider the following stochastic nonconvex matrix optimization problem:
\begin{equation}
    \min\nolimits_{W\in\mathbb R^{m\times n}}
    f(W)=\mathbb E_{\xi\sim\mathcal D}[F(W,\xi)],
    \label{eq:problem}
\end{equation}
where $\xi\in\Xi$ is a data sample drawn from the distribution
$\mathcal D$, and
$F:\mathbb R^{m\times n}\times\Xi\to\mathbb R$ is a differentiable
sample loss with stochastic gradient $\nabla_W F(W,\xi)$.

\subsection{Practical Muon algorithm}

To solve \eqref{eq:problem}, we consider Muon
(Algorithm~\ref{alg:muon}), which combines Nesterov-style momentum,
Newton--Schulz (NS) iterations, and decoupled weight decay (WD).
The NS iterations approximately orthogonalize the momentum matrices 
through matrix polynomials, avoiding an explicit SVD.
We consider two choices of the polynomial $p$ in Definition~\ref{def:ns-polynomials},
 corresponding to practical and standard NS, and refer to Muon with practical NS as practical Muon.

\begin{algorithm}[h]
\caption{Muon (with Newton--Schulz iterations and decoupled weight decay)}
\label{alg:muon}
\begin{algorithmic}[1]

\State \textbf{Input:} Iteration horizon $T\in\mathbb N$,
mini-batch size $b_{\mathrm{mb}}\in\mathbb N$,
momentum coefficient $\beta\in[0,1)$,
normalization floor $\varepsilon>0$, 
NS polynomial $p$,
stepsize $\eta>0$, 
and WD coefficient $\lambda\ge0$

\State \textbf{Initialization:}
parameter $W_0\in\mathbb R^{m\times n}$,
and momentum buffer $B_{-1}=0$

\For{$t=0,\ldots,T-1$}
    \State Independently draw $\xi_{t,1},\ldots,\xi_{t,b_{\mathrm{mb}}}\sim\mathcal D$,
    and compute
    $G_t\gets b_{\mathrm{mb}}^{-1}
    \sum_{i=1}^{b_{\mathrm{mb}}}\nabla_W F(W_t,\xi_{t,i})$

  \State $B_t\gets\beta B_{t-1}+(1-\beta)G_t$
      \hfill
      \raisebox{-0.5\baselineskip}[0pt][0pt]{%
          \algorithmiccomment{Nesterov-style momentum}%
      }
  \State $M_t\gets\beta B_t+(1-\beta)G_t$

    \State $\tau_t\gets\max\{\normF{M_t},\varepsilon\}$\footnotemark,
    \quad $X_{t,0}\gets M_t/\tau_t$

    \For{$j=0,\ldots,4$}
        \State
        $X_{t,j+1}\gets p(X_{t,j}X_{t,j}^\top)X_{t,j}$
                \Comment{Newton--Schulz iterations (Line 7--11)}
    \EndFor

    \State $D_t\gets X_{t,5}$,

    \State
    $W_{t+1}\gets(1-\eta\lambda)W_t-\eta D_t
    =W_t-\eta H_t$
    \Comment{Decoupled weight decay}
\EndFor

\end{algorithmic}
\end{algorithm}

\begin{definition}[Practical and standard Newton--Schulz polynomials]
\label{def:ns-polynomials}
The practical NS polynomial is
\[
    p_{\mathrm P}(x)=a_{\mathrm P}+b_{\mathrm P}x+c_{\mathrm P}x^2,
    \qquad
    (a_{\mathrm P},b_{\mathrm P},c_{\mathrm P})
    =(3.4445,-4.7750,2.0315).
\]
For $k\in\mathbb N$, the degree-$k$ standard NS polynomial
is the Taylor polynomial of $x^{-1/2}$ at $x=1$:
\[
    p_{\mathrm S}^{(k)}(x)
    =\sum\nolimits_{\ell=0}^{k}
      \frac{(2\ell)!}{4^\ell(\ell!)^2}(1-x)^\ell.
\]
In particular,  
$
p_{\mathrm S}^{(2)}(x)
=a_{\mathrm S}+b_{\mathrm S}x+c_{\mathrm S}x^2$ with
$(a_{\mathrm S},b_{\mathrm S},c_{\mathrm S})
=(15/8,-5/4,3/8).
$
\end{definition}

In Algorithm~\ref{alg:muon}, the NS polynomial is chosen as either
$p_{\mathrm P}$ or $p_{\mathrm S}^{(k)}\hspace{-0.3em}.$
For notational simplicity, we denote either choice by $p$,
write $a,b,c$ for its coefficients of $1,x,x^2$, respectively.
Moreover, define
$
q(s)=sp(s^2),~
q_0(s)=s,~
q_{j+1}(s)=q(q_j(s)),~j=0,\ldots,4
$, and write
 $(q_{\mathrm P},q_{\mathrm P,j})$ and
$(q_{\mathrm S}^{(k)},q_{\mathrm S,j}^{(k)})$ for $(q,q_j)$
under practical and standard NS, respectively.
The following lemma characterizes how these polynomials are used in NS iterations to approximate matrix orthogonalization.

\footnotetext{Another common normalization uses
$\normF{M_t}+\varepsilon$.  The two denominators differ by
$\min\{\normF{M_t},\varepsilon\}\le\varepsilon$.
We adopt the maximum form, consistent with the PyTorch implementation.}

\begin{lemma}[Approximate orthogonalization via NS;
\citealp{kim2026muon}, Proposition~1 and Lemma~1]
\label{lem:ns-polynomials}
 \textbf{(i) NS output.}
    In Algorithm~\ref{alg:muon}, write the compact SVD
    $M_t=U_t\operatorname{diag}(\sigma_{1,t},\ldots,\sigma_{r_t,t})V_t^\top$,
    with $r_t=\operatorname{rank}(M_t)$, and 
    $s_{i,t}=\sigma_{i,t}/\tau_t\in(0,1]$.
The five NS steps preserve the left and right singular-vector
matrices $U_t$ and $V_t$, and yield
\[
D_t=X_{t,5}
=U_t\operatorname{diag}\bigl(
q_5(s_{1,t}),\ldots,q_5(s_{r_t,t})
\bigr)V_t^\top.
\]
\textbf{(ii) Approximation error.}
The singular-value map $q_{\mathrm P}$ of one practical NS step
is nonmonotone on $[0,1]$ and does not fix $1$. Specifically,
\[
q_{\mathrm P}'(0)=3.4445>0,\qquad
q_{\mathrm P}'(1)=-0.7230<0,\qquad
q_{\mathrm P}(1)=0.7010\ne1.
\]
In contrast, for every $k\ge1$, the singular-value map
$q_{\mathrm S}^{(k)}$ of one standard NS step is nondecreasing
on $[0,1]$ and fixes $1$.
These properties, together with the Taylor construction, yield
\begin{equation}
\|X_{t,j}-U_tV_t^\top\|_{\mathrm{op}}
\le
(1-\min\nolimits_{1\le i\le r_t}s_{i,t}^{\,2})^{(k+1)^j},
\qquad j=0,\ldots,5.
\label{standard error}
\end{equation}
Thus, each standard NS step approximates exact orthogonalization with convergence order $k+1$.
\end{lemma}

\paragraph{Challenges in analyzing practical Muon.}
The role of NS in Muon is to approximate the orthogonalized
momentum matrix $U_tV_t^\top$ without explicitly computing an SVD.
This amounts to preserving the singular vectors while transforming each positive
singular value toward $1$. The normalization in
Algorithm~\ref{alg:muon} places the initial singular values in $(0,1]$.
For standard NS, monotonicity and the fixed point at $1$ yield the
superlinear error bound in \eqref{standard error}. Practical NS
lacks both properties, 
making its approximation
error hard to control and Muon's convergence difficult
to analyze.

\subsection{Assumptions}

Our analysis is based on the following assumptions, which are standard in stochastic nonconvex optimization \citep[e.g.,][]{ghadimi2013stochastic,li2026frac,shen2026on,choudhury2026muon}.

\begin{assumption}[Low boundedness]
\label{ass:lower} 
The objective satisfies
$
 f_{\inf}=\inf_{W\in\mathbb R^{m\times n}}f(W)>-\infty.
$
\end{assumption}

\begin{assumption}[Smoothness]
\label{ass:smooth}
There exists a constant $L>0$ such that, for all $X,Y\in\mathbb R^{m\times n}$,
\[
    \|\nabla f(X)-\nabla f(Y)\|_F
    \le L\|X-Y\|_F.
\]
\end{assumption}

\begin{assumption}[Stochastic gradient oracle]
    \label{ass:noise}
    \hspace{-0.3em}
There exists a constant $\sigma\hspace{-0.1em}\ge\hspace{-0.1em}0$ such that, for all $W\hspace{-0.35em}\in\hspace{-0.2em}\mathbb R^{\hspace{-0.1em}m\hspace{-0.15em}\times \hspace{-0.1em}n}$\hspace{-0.25em},
\[
\mathbb E_{\xi}[\nabla F_W(W,\xi)]=\nabla f(W),
\qquad
\mathbb E_{\xi}\!\left[
\|\nabla F_W(W,\xi)-\nabla f(W)\|_F^2
\right]\le \sigma^2.
\]
For any mini-batch of size $b_{\mathrm{mb}}$,
the averaged gradient has variance at most $\sigma^2/b_{\mathrm{mb}}$.
\end{assumption}

\section{Optimization interpretation} \label{sec:interpret}

\paragraph{Main viewpoint.}
We show that Muon with NS and WD, using either practical or standard NS,
admits a unified interpretation as right-preconditioned optimization
on a dynamic objective consisting of $f$ and a weighted $\ell_2$ regularizer, with exact equivalence
when the momentum is omitted.

\paragraph{Dynamic objective.}
We seek an objective for which the Muon update provides a descent
direction. A natural idea is to examine the fixed point of Algorithm~\ref{alg:muon} (with $m=n=1$ for illustration
, allowing us to write $D_t=q_5(M_t/\tau_t)$ without an SVD representation), where $M_t=B_t=G_t$ and
\[
D_t+\lambda W_t=q_5(G_t/\tau_t)+\lambda W_t=0.
\]
Inverting $\hspace{-0.1em}q_5\hspace{-0.1em}$ would allow us to isolate the gradient
and identify a corresponding objective.
For practical NS, however, both $\hspace{-0.05em}q_{\mathrm P}\hspace{-0.05em}$ and
$\hspace{-0.05em}q_{\mathrm P,5}\hspace{-0.05em}$ are nonmonotone, so the latter admits no global
inverse, obstructing this direct approach.
Motivated by the preconditioning interpretation of Muon with exact orthogonalization
\citep{ma2026preconditioning}, beyond fixed points,
we seek a right preconditioner $P_t$ such that
\[D_t=M_tP_t,~~\text{i.e.,}~~
U_t\operatorname{diag}\bigl(
q_5(\sigma_{1,t}/\tau_t),\ldots,q_5(\sigma_{r_t,t}/\tau_t)
\bigr)V_t^\top \hspace{-0.2em}=\hspace{-0.1em} U_t\operatorname{diag}\bigl(
\sigma_{1,t},\ldots,\sigma_{r_t,t})V_t^\top P_t.
\]
This determines $ P_t $ on the span of $ V_t\hspace{-0.05em}$,
 while the WD term $\lambda W_t $ suggests a $\hspace{-0.05em}P_t^{\hspace{-0.025em}-\hspace{-0.1em}1}\hspace{-0.1em}$ weighted
$ \ell_2$\hspace{-0.05em} regularizer, which requires $P_t$ to be invertible.
Thus, we extend $P_t$ positively to the nullspace of $M_t$
and define
\begin{equation}
\begin{aligned}
    P_t
    &=V_t\operatorname{diag}\Big(
      \frac{\chi_5(s_{1,t})}{\tau_t},\ldots,
      \frac{\chi_5(s_{r_t,t})}{\tau_t}
      \Big)V_t^\top\ +\frac{\chi_5(0)}{\tau_t}
      (I_n-V_tV_t^\top),
\end{aligned}
\label{eq:right-preconditioner}
\end{equation}
where $s_{i,t}\in(0,1]$, and $\chi_5(s)=q_5(s)/s>0$ for $s\in(0,1]$,
with the continuous extension $\chi_5(0)=a^5$.
The resulting dynamic objective is
\begin{equation}
\Phi_t(W)=f(W)+\mathcal R_t(W),
\qquad
\mathcal R_t(W)
=\frac{\lambda}{2}\operatorname{tr}(WP_t^{-1}W^\top).
\label{eq:dynamic-objective}
\end{equation}
Here, $P_t$ is treated as a fixed
matrix when differentiating $\Phi_t$ in $W$.
The following proposition establishes the exact equivalence between the Muon update
and right-preconditioned optimization of the dynamic objective
$\Phi_t$.

\begin{proposition} \label{prop:dynamic-regularization}
\textbf{(i) Properties of the right preconditioner and regularizer.}
For both practical and standard NS, the preconditioner $P_t$
in \eqref{eq:right-preconditioner}  
satisfies
\begin{align*}
    P_t\succ0,\qquad D_t=M_tP_t,\qquad \langle M_t, D_t \rangle_F = \|M_tP_t^{1/2}\|_F^2\ge0.
\end{align*}
Thus, $D_t$ is positively aligned with the momentum direction $M_t$
whenever $M_t\ne0$.
Moreover, for $\lambda>0$, the regularizer $\mathcal R_t$
in \eqref{eq:dynamic-objective} is
$\lambda\tau_t/C_1$-strongly convex, where
$C_1=\max_{s\in[0,1]}\chi_5(s)>0$.

\textbf{(ii) Exact right-preconditioned equivalence.}
Ignore the momentum by setting $\beta=0$, Muon update in Algorithm~\ref{alg:muon} is exactly one
right-preconditioned stochastic gradient descent step on the
dynamic objective \eqref{eq:dynamic-objective}, i.e., 
\[
W_{t+1}=W_t-\eta\widehat{\nabla}\Phi_t(W_t)P_t,
\qquad
\widehat{\nabla}\Phi_t(W_t)
=G_t+\lambda W_tP_t^{-1}.
\]
If Assumption~\ref{ass:noise} futher holds, we have
$
\mathbb E[
\widehat{\nabla}\Phi_t(W_t)-\nabla\Phi_t(W_t)
]=0.
$
\end{proposition}

\begin{proof}
    Recall that $q_0(s)=s$ and
$q_{j+1}(s)=q(q_j(s))=q_j(s)p(q_j(s)^2)$ for $j=0,\ldots,4$.
For practical NS, $p_{\mathrm P}(x)>0$ for all $x\in\mathbb R$, since $ c_{\mathrm P}\hspace{-0.1em}>\hspace{-0.1em}0 $ and
$ b_{\mathrm P}^2\hspace{-0.1em}-\hspace{-0.1em}4a_{\mathrm P}c_{\mathrm P}\hspace{-0.1em}<\hspace{-0.1em}0\hspace{-0.05em}$.
For standard NS, whenever $s\in[0,1]$, Lemma~\ref{lem:ns-polynomials} gives
$q_{\mathrm S}^{(k)}(s)\in[0,1]$, 
$p_{\mathrm S}^{(k)}(s)
    =1+\sum_{\ell=1}^{k}
    \frac{(2\ell)!}{4^\ell(\ell!)^2}(1-s)^\ell>0$ , and hence $p(q_{\mathrm S}^{(k)}(s))\hspace{-0.1em}>\hspace{-0.1em}0$.
As a result, for either NS choice, applying the recurrence
over all five NS steps gives
\[
    \chi_5(s)
    =\prod\nolimits_{j=0}^{4}p(q_j(s)^2)>0,
    \qquad s\in(0,1].
\]
The eigenvalues of $P_t$ are either 
$\chi_5(s_{i,t})/\tau_t$ with $s_{i,t}\in(0,1]$, or $\chi_5(0)/\tau_t$, giving
\begin{equation}
  \frac{C_2}{\tau_t}I_n\preceq P_t
  \preceq\frac{C_1}{\tau_t}I_n,   \qquad C_2=\min_{s\in[0,1]}\chi_5(s)>0.     \label{eq:P-bounds}
\end{equation}   
Thus $P_t\succ0$, and the strong convexity of $\mathcal R_t$
follows from $P_t^{-1}\succeq(\tau_t/C_1)I_n$.
By Lemma~\ref{lem:ns-polynomials} and $V_t^\top(I_n-V_tV_t^\top)=0$,  
 we have
\[
    D_t
    =U_t\operatorname{diag}\!\left(
      q_5(s_{1,t}),\ldots,q_5(s_{r_t,t})
      \right)V_t^\top
    =M_tP_t,
\]
which implies
$\langle M_t,D_t\rangle_F=\|M_tP_t^{1/2}\|_F^2\ge0$.
When $\beta=0$, we have $M_t=G_t$, and
\[
    \widehat{\nabla}\Phi_t(W_t)P_t
    =G_tP_t+\lambda W_t
    =D_t+\lambda W_t=H_t,
\]
proving the update identities.
Holding $P_t$ fixed when differentiating gives
$\nabla\mathcal R_t(W)=\lambda WP_t^{-1}$, and
Assumption~\ref{ass:noise} directly implies
$
\mathbb E~[
\widehat{\nabla}\Phi_t(W_t)-\nabla\Phi_t(W_t)]=0.
$
\end{proof}

\paragraph{Recovery of the original problem.}
Consider the deterministic setting where $G_t=\nabla f(W_t)$.
By the definition of $\mathcal R_t$ in \eqref{eq:dynamic-objective}
and the bounds \eqref{eq:P-bounds}, we have
\begin{equation}
    0\le\mathcal R_t(W)
    \le\frac{\lambda\tau_t}{2C_2}\nf{W}^2,
    \qquad
    \nf{\nabla\mathcal R_t(W)}
    \le\frac{\lambda\tau_t}{C_2}\nf{W}. 
    \label{eq:regularizer-vanishes}
\end{equation}
\looseness=-1
If $\nf{\nabla f(W_t)}\to0$, then $\nf{M_t}\to0$ and hence
$\tau_t=\max\{\nf{M_t},\varepsilon\}\to\varepsilon$,
where $\varepsilon$ is chosen to be very small in practice,
e.g., $10^{-7}$ in the official PyTorch implementation.
 Together with the boundedness of $\{W_t\}$, which WD naturally
promotes,
  inequalities \eqref{eq:regularizer-vanishes} show that  
both $\mathcal R_t(W_t)$ and $\nf{\nabla\mathcal R_t(W_t)}$ are
$\mathcal O(10^{-7})$, and the dynamic objective reduces to the original
objective up to numerical precision.
Independently of this interpretation, our main theorem in the next
section allows for stochastic setting and directly
establishes stationarity for the original stochastic nonconvex problem.


\section{Main results} \label{sec:main:res:rate}

For Muon with practical or standard NS and WD, 
we establish the following convergence guarantees in the stochastic nonconvex setting.

\begin{theorem}[NS with WD]\label{thm:main}
Under Assumptions~\ref{ass:lower}--\ref{ass:noise},
consider the iterates of Algorithm~\ref{alg:muon} with
\begin{equation}
 \eta=\frac{c_1}{T^{3/4}},
 ~1-\beta=\frac{c_2}{T^{1/2}},~
  \varepsilon \in \left(0,\frac{c_3}{T^{1/4}}\right],~
 \lambda=
 \frac{ C_2}
 {2\bigl(\nf{W_0}+ C_3c_1T^{1/4}\bigr)},~
 T\ge\max\{1,c_2^2\},               \label{eq:parameters}
\end{equation}
where $c_1,c_2,c_3$ are $\mathcal O(1)$ constants defined in
\eqref{eq:balanced-constants}, and $C_3$ is an $\mathcal O(1)$
constant defined in \eqref{eq:kappa}.
Then 
\begin{align}
 \frac1T\sum\nolimits_{t=0}^{T-1}\E\nf{\nabla f(W_t)}
 =
    \mathcal{O}\Big(\frac{\sqrt[4]{L(f(W_0)-f_{\inf})\sigma^2/b_{\mathrm{mb}}}}{{T^{1/4}}}
    \Big).
                  \label{eq:rate}
\end{align}
\end{theorem}

\paragraph{Discussion: practical NS with WD.}
\begin{itemize}[leftmargin=*, labelsep=0.5em]
    \item \textbf{\textit{Convergence guarantee.}}
    To the best of our knowledge, Theorem~\ref{thm:main} provides
    the first convergence guarantee for practical Muon in the
    stochastic nonconvex setting, based on a direct analysis of the
    practical NS iteration and its interaction with WD.
    While \citet{qian2026convergence} also cover practical NS with WD,
    their guarantees rely on star-convexity and parameter choices that
    depend on the unknown norm $\|W^*\|$ of a global minimizer $W^*$.
Other analyses either use exact orthogonalization when accounting for WD \citep{pethick2025training,chen2026muon}, 
or omit WD while studying practical NS \citep{do2026muon,qian2026convergence}, standard NS \citep{kim2026muon,choudhury2026muon}, or exact orthogonalization \citep{shen2026on}.

    \item \textbf{\textit{Comparison with AdamW and SGD.}}
    The theoretical convergence rate in \eqref{eq:rate} improves dimension dependence by
    a factor of $\sqrt d$ over the Frobenius norm
    guarantee implied by the best-known AdamW result given by 
    \eqref{rate:adamW} \citep{li2026frac} in the stochastic nonconvex setting.
    Specifically, AdamW's entrywise $\ell_1$ guarantee contains
    an additional $\sqrt d$ factor that persists after conversion
    to the Frobenius norm.
    Moreover, our rate matches that of SGD in the Frobenius norm \citep{ghadimi2013stochastic},
    and that of AdamW in the entrywise $\ell_1$ norm.

    \item \textbf{\textit{Parameter choices.}}
    The choice $\lambda\hspace{-0.1em}=\hspace{-0.1em}\mathcal{O}(T^{-1/4})$ is reasonable for practical
    training horizons. For example, Moonlight pretraining with Muon takes approximately
    $1.76\times10^5$ updates \citep{liu2025muon},
    giving $T^{-1/4}\approx0.05$.
    This mild dependence on $T$ permits a non-negligible WD
    coefficient that retains its regularization role.
    The theoretical condition on $\varepsilon$ is also mild,
    since commonly used values, such as the PyTorch default $10^{-7}$,
    are far below the corresponding scale $T^{-1/4}$.
    Moreover, the condition on $T$ imposes only an $\mathcal O(1)$
    lower threshold.
    Small stepsize $\eta$ and momentum coefficient
    $\beta$ close to one are common in stochastic nonconvex
    convergence analyses \citep{ghadimi2013stochastic,li2026frac}.
    Finally, the problem dependent choices of $c_1,c_2,c_3$ only
    improve the multiplicative factor in the convergence rate, and any positive values retain
    the $\mathcal O(T^{-1/4})$ rate.
\end{itemize}

\paragraph{Discussion: standard NS with WD.}

The comparisons with AdamW and SGD and the discussion of
parameter choices above also apply to standard NS with WD.

\begin{itemize}[leftmargin=*, labelsep=0.5em]

    \item \textbf{\textit{Convergence guarantee.}}
    To the best of our knowledge, Theorem~\ref{thm:main} provides
    the first convergence guarantee for Muon with standard NS and WD
    in the stochastic nonconvex setting.
    In contrast, the existing analysis requires both star-convexity and
    parameter choices that depend on the norm $\|W^*\|$
    of an unknown global minimizer \citep{qian2026convergence}, while
    the proof of Proposition~3.3 in \citet{sato2025convergence}
    contains an error identified by \citet{apte2026scale}.
    The counterexample in Apte's Proposition~3.2
    also invalidates Sato et al.'s Proposition~3.3
    when standard NS is used.

\item \textbf{\textit{Convergence rate.}}
    By analyzing practical and standard NS in a unified framework,
    we obtain the rate in \hspace{-0.035em}\eqref{eq:rate}\hspace{-0.05em} independent of
    the orthogonalization error of standard NS,
    which can lead to a very large multiplicative factor
    in existing convergence rates \citep{kim2026muon,choudhury2026muon}.
    For example, \citet{kim2026muon} obtain the factor
$(1\hspace{-0.1em}-\hspace{-0.1em}\epsilon_5)^{-1}$ for standard NS without WD, where
    \[
        \epsilon_{t,5}
        =\|X_{t,5}-U_tV_t^\top\|_{\mathrm{op}}=1-q_5\big(\min\nolimits_{1\le i\le r_t} s_{i,t}\big),
        \qquad
        \epsilon_5=\max\nolimits_{0\le t<T}\epsilon_{t,5}.
    \]
    Despite the superlinear decay in \eqref{standard error},
    this factor can be large for small normalized singular values.
    For a quantitative illustration,
    \citet[Figure~1(c)]{wu2026achieving}
    show a Muon momentum spectrum in a LLaMA attention layer
    spanning roughly eight or more orders of magnitude.
    For such an NS input, five degree-$2$ standard NS steps
    would give $(1-\epsilon_5)^{-1}>10^6$.

\end{itemize}

The following corollary gives the corresponding convergence guarantee
without WD.

\begin{corollary}[NS without WD]\label{cor:no-wd}
Under Assumptions~\ref{ass:lower}--\ref{ass:noise},
consider the iterates of Algorithm~\ref{alg:muon} with
\[
 \eta=\frac{c_1}{T^{3/4}},\quad
 1-\beta=\frac{c_2}{T^{1/2}},\quad
 \varepsilon\in\left(0,\frac{c_3}{T^{1/4}}\right],\quad
 \lambda=0,\quad
 T\ge\max\{1,c_2^2\},
\]
where $c_1,c_2,c_3$ are defined in
\eqref{eq:balanced-constants} with $(C_2/2+C_3)$ replaced by $C_3$.
Then
\begin{equation}
 \frac1T\sum\nolimits_{t=0}^{T-1}\E\nf{\nabla f(W_t)}
=  \mathcal{O}\Big(\frac{\sqrt[4]{L(f(W_0)-f_{\inf})\sigma^2/b_{\mathrm{mb}}}}{{T^{1/4}}}
    \Big).
 \label{eq:no-wd-rate}
\end{equation}
\end{corollary}

\paragraph{Discussion: practical and standard NS without WD.}

Corollary~\ref{cor:no-wd} yields a convergence rate similar to that with WD.
The Muon rate \eqref{eq:no-wd-rate} improves dimension dependence
by a factor of $\sqrt d$ over the AdamW Frobenius norm rate
implied by the best-known result \eqref{rate:adamW},
and does not depend on the worst-case orthogonalization error.
Regarding the choice of NS polynomial, the corollary does not
reveal a clear difference in convergence behavior between
practical and standard NS.

\section{Convergence analysis} \label{sec:proof}

We first state the key insight for our analysis: quantifying
the alignment of NS directions, particularly those produced by practical NS,
with the momentum direction, which supports their use as effective
descent directions.
We then prove convergence, 
with a few routine details deferred to Appendix~\ref{app:proof}.

\subsection{Key insight}

The favorable alignment of practical and standard NS directions
with the momentum is key to Muon's convergence, and is quantified
in the following lemma.

\begin{lemma}[NS direction alignment]\label{lem:ns}
For both practical and standard NS,
the momentum $M_t$ and the resulting
direction $D_t$ in Algorithm~\ref{alg:muon} satisfy
{
\setlength{\abovedisplayskip}{5pt}
\setlength{\belowdisplayskip}{5pt}
\[
\langle M_t,D_t\rangle_F
\ge C_2\nf{M_t}-C_2\varepsilon/4,
\qquad
\nf{D_t}\le C_3,
\]
}\noindent
where $C_2\hspace{-0.1em}=\hspace{-0.1em}\min_{s\in[0,1]}\chi_5(s)\hspace{-0.1em}>\hspace{-0.1em}0,~C_1\hspace{-0.1em}=\hspace{-0.1em}\max_{s\in[0,1]}\chi_5(s)\hspace{-0.1em}>\hspace{-0.1em}0$,
recalled here for convenience, and
\begin{align}
C_3=\min\{C_1,
q_{\max}\sqrt{r}\},\qquad
q_{\max}=\max\nolimits_{s\in[0,1]}q_5(s),\qquad
r=\min\{m,n\}.
 \label{eq:kappa}
\end{align}
 In particular, 
 \vspace{-0.5em}
\begin{align} \label{constant:particular}
    (C_2,C_1,q_{\max})=
    \begin{cases}
    (0.6964\ldots,\,484.8762\ldots,\,1.2023\ldots),
        & \text{practical NS},\\
    (1,\,a^5,\,1),
        & \text{standard NS}.
    \end{cases}
\end{align}
\end{lemma}

\begin{proof}

The proof of Proposition~\ref{prop:dynamic-regularization}
ensures that $\chi_5$, $C_1$, $C_2$, and $q_{\max}$ are
well defined.
 By definition,
\begin{align}
    C_2s\le q_5(s)\le C_1s,
    \qquad
    0\le q_5(s)\le q_{\max},
    \qquad s\in[0,1]. \label{eq:both}
\end{align}
We then prove the bounds in Lemma~\ref{lem:ns}, 
with sharp constants $C_1$, $C_2$, and $q_{\max}$ determined in Appendix~\ref{app:proof}.
By Lemma~\ref{lem:ns-polynomials},
$D_t=U\operatorname{diag}(q_5(s_{i,t}),\ldots,q_5(s_{r_t,t})V^\top$.
Then \eqref{eq:both} yields
\[
 \langle M_t,D_t \rangle_F
 =\tau_t\textstyle\sum\nolimits_i s_{i,t}q_5(s_{i,t})
 \ge C_2\tau_t\textstyle\sum\nolimits_i s_{i,t}^2= C_2\nf{M_t}^2/\tau_t,
\]
where
 $\nf{M_t}^2/\tau_t\ge \nf{M_t}-\varepsilon/4$, since for $\nf{M_t}\le\varepsilon$ the inequality reduces to
$(\nf{M_t}-\varepsilon/2)^2/\varepsilon\ge0$, whereas for $\nf{M_t}\ge\varepsilon$, the left-hand side equals $\nf{M_t}$.
Moreover, \eqref{eq:both} together with $\sum_i s_{i,t}^2=\nf{M_t}^2/\tau_t^2\le1$ gives
 \[
 \nf{D_t}=\sqrt{\textstyle\sum\nolimits_i q_5(s_{i,t})^2}
 \le\min\Big\{C_1\sqrt{\textstyle\sum\nolimits_i s_{i,t}^2},
 q_{\max}\sqrt{r}\Big\}
 \le C_3. \qedhere
\] 
 \end{proof}

 \paragraph{Discussion.}
Lemma~\ref{lem:ns} establishes favorable alignment of both practical
and standard NS directions with the momentum, up to a small
$\hspace{-0.05em}O(\hspace{-0.0em}\varepsilon\hspace{-0.05em})\hspace{-0.05em}$ term, with $\varepsilon\hspace{-0.25em}=\hspace{-0.275em}10^{\hspace{-0.05em}-\hspace{-0.05em}7}$ by default in PyTorch.
This estimate is independent of the orthogonalization error,
which is particularly useful for practical NS since its singular-value map
is nonmonotone and does not fix $1$, precluding asymptotically exact orthogonalization, 
let alone superlinear convergence of standard NS.
Moreover, the lemma 
establishes a uniform bound on the Frobenius norm of the NS direction, $\hspace{-0.1em}\|D_t\|_F\hspace{-0.1em}\le\hspace{-0.1em} C_3\hspace{-0.1em}\le\hspace{-0.1em} C_1\hspace{-0.1em}=\hspace{-0.1em}\mathcal O(1)$,
preventing unbounded growth with $r$ and limiting
the rank dependence of the resulting convergence rate.
Together, alignment and boundedness form the basis of our descent analysis
of the original objective.
 
\subsection{Convergence proof}

We now prove Theorem~\ref{thm:main}, deferring a few routine
details of the momentum analysis to Appendix~\ref{app:proof}.
Corollary~\ref{cor:no-wd} follows from the same argument
by omitting the WD term, i.e., setting $\lambda=0$.

\begin{proof}
    
The update 
$W_{t+1}=W_t-\eta H_t=W_t-\eta(D_t+\lambda W_t)$ and
Assumption~\ref{ass:smooth} gives
\begin{equation}
 f(W_{t+1})
 \le f(W_t)
 -\eta\langle\nabla f(W_t),H_t\rangle_F
 +L\eta^2/2\nf{H_t}^{2}.                          \label{eq:smoothness-step}
\end{equation}
Let $e_t=M_t-\nabla f(W_t)$. The linear term decomposes as
\begin{equation}
 \langle\nabla f(W_t),H_t\rangle_F
 =\langle M_t,D_t\rangle_F-\langle e_t,D_t\rangle_F
 +\lambda\langle\nabla f(W_t),W_t\rangle_F.                \label{eq:linear-decomposition}
\end{equation}

\paragraph{Step 1. Alignment of the NS directions with momentum.}
We use Lemma~\ref{lem:ns} to control the alignment
$\langle M_t,D_t\rangle_F$ and the norm $\nf{D_t}$.

\vspace{-0.5em}
\paragraph{Step 2. Bounds on \texorpdfstring{$\lambda W_t$ and $H_t$}{lambda Wt and Ht}}

From \eqref{eq:parameters}, \eqref{eq:kappa}, and \eqref{constant:particular},
$
 0<\eta\lambda
 \le C_2/(2 C_3T)<1.  $
Lemma~\ref{lem:ns} and the update in Algorithm~\ref{alg:muon}
therefore imply
\[
 \nf{W_{t+1}}
 \le(1-\eta\lambda)\nf{W_t}+\eta C_3
 \le\nf{W_t}+\eta C_3.
\]
For $0\le t\le T$, \eqref{eq:parameters} gives
\begin{equation}
\lambda\nf{W_t}
 \le\lambda\bigl(\nf{W_0}+ C_3\eta t\bigr)
 \le C_2/2, \qquad
 \nf{H_t}\le C_4=C_2/2+C_3.
                        \label{eq:direction-control}
\end{equation}
We can now bound the linear term  \eqref{eq:linear-decomposition}.
Lemma~\ref{lem:ns}, \eqref{eq:direction-control}, and
$\nf{M_t}\ge\nf{\nabla f(W_t)}-\nf{e_t}$ yield
\begin{align}
 \langle\nabla f(W_t),H_t\rangle_F
 &\ge  C_2\nf{M_t}- C_2\varepsilon/4
      - C_3\nf{e_t}-\lambda\nf{W_t}\nf{\nabla f(W_t)} \notag\\
 &\ge \bigl( C_2-\lambda\nf{W_t}\bigr)\nf{\nabla f(W_t)}
      -( C_2+ C_3)\nf{e_t}
      -C_2\varepsilon/4 \notag\\
 &\ge \frac{ C_2}{2}\nf{\nabla f(W_t)}
      -( C_2+ C_3)\nf{e_t}
      -\frac{ C_2\varepsilon}{4}.                       \label{eq:alignment}
\end{align}
Substituting \eqref{eq:alignment} and \eqref{eq:direction-control} into
\eqref{eq:smoothness-step} gives
\begin{equation}
 f(W_{t+1})
 \le f(W_t)-\frac{ C_2\eta}{2}\nf{\nabla f(W_t)}
 +\eta( C_2+ C_3)\nf{e_t}
 +\frac{ C_2\eta\varepsilon}{4}
 +\frac{L\eta^2 C_4^2}{2}.                                 \label{eq:descent}
\end{equation}
\vspace{-1.5em}
\paragraph{Step 3. Tracking error of the Nesterov-style momentum.}
\begin{lemma}[Momentum tracking]\label{lem:tracking}
Under the assumptions and parameter choices of Theorem~\ref{thm:main},
\begin{align}
 \frac1T\hspace{-0.1em}\sum\nolimits_{t=0}^{T-1}\hspace{-0.1em}\E\nf{M_t\hspace{-0.1em}-\hspace{-0.1em}\nabla f(W_t)}
 \le{}&\frac{\nf{B_{-1}\hspace{-0.1em}-\hspace{-0.1em}\nabla f(W_0)}}{c_2}T^{-1/2}
 \hspace{-0.1em}+\hspace{-0.1em}\Big(
 \frac{Lc_1 C_4}{c_2}
 \hspace{-0.1em}+\hspace{-0.1em}\sqrt{2c_2}\,\frac{\sigma}{\sqrt{b_{\mathrm{mb}}}}
 \Big)T^{-1/4}\hspace{-0.1em}.                                        \nonumber
\end{align}
\end{lemma}
Summing \eqref{eq:descent}, taking expectations, and using Lemma~\ref{lem:tracking}, \eqref{eq:parameters}, and $f(W_t)\ge f_{\inf}$, we obtain
\begin{align}
 \frac1T\hspace{-0.2em}\sum_{t=0}^{T-1}\hspace{-0.1em}\E\hspace{-0.1em}\nf{\nabla f(W_t)}
\hspace{-0.15em}&\le\hspace{-0.25em}
\Big[\frac{2(f(W_0)\hspace{-0.1em}-\hspace{-0.1em}f_{\inf})}{C_2c_1}
   \hspace{-0.1em}+\hspace{-0.1em}\frac{2(C_2\hspace{-0.1em}+\hspace{-0.1em}C_3)LC_4c_1}{C_2c_2}  
   \hspace{-0.1em}+\hspace{-0.1em}\frac{2\sqrt{2}(C_2\hspace{-0.1em}+\hspace{-0.1em}C_3)\sqrt{c_2}\sigma}{C_2\sqrt{b_{\mathrm{mb}}}}   \hspace{-0.1em}+\hspace{-0.1em}\frac{c_3}{2}
 \Big]T^{-\hspace{-0.1em}1/4} \notag\\
 &\quad 
    +\frac{2( C_2+ C_3)}{ C_2c_2}
   \nf{B_{-1}-\nabla f(W_0)}\,T^{-1/2}
 +\frac{Lc_1 C_4^2}{ C_2}T^{-3/4}.                     \label{eq:pre-balanced-rate}
\end{align}
For fixed $c_2$, the first two terms in the bracket are minimized by
$c_1^2=(f(W_0)-f_{\inf})c_2/[( C_2+ C_3)L C_4]$.
Then optimizing over $c_2$ and choosing $c_3$ to match the
resulting statistical scale gives
\vspace{-0.25em}
\begin{align}
    &\qquad\qquad\qquad\quad  c_1=\frac{2^{1/4}(f(W_0)-f_{\inf})^{3/4}b_{\mathrm{mb}}^{1/4}}
 {( C_2+ C_3)^{3/4}(C_2/2 + C_3)^{1/4}(L \sigma^2)^{1/4}},\nonumber\\
 &
    c_2=\Big(
 \frac{2(C_2/2 + C_3)L(f(W_0)-f_{\inf})b_{\mathrm{mb}}}{( C_2+ C_3)\sigma^2}
 \Big)^{1/2},~
 c_3=\Big(\frac{L(f(W_0)-f_{\inf})\sigma^2}{b_{\mathrm{mb}}}\Big)^{1/4}.
 \label{eq:balanced-constants}
\end{align}\vspace{-0.25em}
For these choices, the bracket in \eqref{eq:pre-balanced-rate} equals
\[
 \Big[
   2^{11/4}\Big(1 + \frac{C_3}{C_2}\Big)^{3/4}\Big(\frac{1}{2} + \frac{C_3}{C_2}\Big)^{1/4} 
  +\frac12
 \Big]\Big(\frac{L(f(W_0)-f_{\inf})\sigma^2}{b_{\mathrm{mb}}}\Big)^{1/4}. \qedhere
\]
\end{proof}

\section{Advantage of practical NS} \label{sec:advPNS}
Our convergence results do not reveal an advantage of practical NS over standard NS.
As discussed for standard NS with WD,
the NS iterations may leave very
 small singular values far from $1$, potentially slowing convergence.
Practical NS was designed to amplify them faster by increasing
the singular-value map's slope at zero \citep{jordan2024muon},
and the following ill-conditioned quadratic illustrates its optimization advantage.
The detailed proof is provided in Appendix~\ref{app:proof}.



\begin{proposition}[Ill-conditioned quadratic]
\label{prop:two-level}
Let $h,\ell,\kappa\in\mathbb N$, with $\kappa\gg1$, $\ell=h\kappa$, $r=h+\ell$.
Define
\[
Q=\operatorname{diag}(\kappa I_h,I_\ell),\qquad
f(W)=\frac12\operatorname{tr}(W^\top QW),\qquad
W_0=I_r.
\]
The high-curvature block accounts for most of the gradient energy,
while each block contributes half of the objective value.
We consider practical NS or standard NS with $k=2$, and normalize
without an $\varepsilon$-stabilizer, stopping when the gradient vanishes.

\textbf{(i) One-step contraction.} Denote 
$
\tau_0\hspace{-0.1em}=\hspace{-0.1em}\nf{\nabla f(W_0)\hspace{-0.1em}}=\hspace{-0.1em}\sqrt{h\kappa(\kappa+1)},~
u=q_5(\kappa/\tau_0),~v=q_5(1/\tau_0).
$
For deterministic Muon without momentum or WD, the respective optimal stepsizes minimizing $f(W_1)$ for the two NS choices yield
\begin{equation}
\frac{f(W_1)}{f(W_0)}
=\frac{h\kappa\ell(u-v)^2}
{(h\kappa+\ell)(h\kappa u^2+\ell v^2)}=
\frac{(u-v)^2}{2(u^2+v^2)}
\le
\left(\frac{u-v}{u+v}\right)^2.  \nonumber
\end{equation}

Consider the example with $h=1$ and $\ell=\kappa=100$, five practical and standard NS
iterations give $(u,v)\approx(0.7020,0.6968)$ and
$(1,0.2274)$, respectively, leaving $0.0014\%$ and $28.3789\%$
of $f(W_0)$ after the one-step update.

For fixed $h$ and $\kappa\to\infty$, substituting
$u=q_5(1/\sqrt h)+\mathcal O(\kappa^{-1})$ and
$v=a^5/(\kappa\sqrt h)+\mathcal O(\kappa^{-2})$ gives
\[
\frac{f(W_1)}{f(W_0)}
=\frac12-\frac{a^5}{\kappa\sqrt h\,q_5(1/\sqrt h)}
+\mathcal O(\kappa^{-2}).
\]
When $q_5(1/\sqrt h)$ is close to one for both NS
choices, the leading-order improvement is mainly governed by $a^5$,
which is approximately $484.8762$ for practical NS and $23.1743$
for standard NS. Thus, practical NS yields a substantially
stronger one-step contraction.

\textbf{(ii) Full optimization trajectory.}
Consider the example with $h=1$ and $\ell=\kappa=100$.
Let $W_t^{\mathrm P}$ and $W_t^{\mathrm S}$ denote the Muon iterates generated by practical and standard NS, respectively, 
using the optimal stepsizes at each iteration $t$ that minimize $f(W_{t+1}^{\mathrm P})$ and $f(W_{t+1}^{\mathrm S})$.
Then
\begin{equation}
\frac{f(W_t^{\mathrm P})}{f(W_0)}\le (6.4514\times10^{-5})^t,
\qquad
\frac{f(W_t^{\mathrm S})}{f(W_0)}\ge(0.1406)^t. \nonumber
\end{equation}
Further comparisons for different values of $h$ and $\kappa$
are provided in Table~\ref{tab:ns-one-step}.
\begin{table}[h]
\centering
\caption{One-step and full trajectory comparisons for the
ill-conditioned quadratic.
Panel (a) reports $f(W_1^{\mathrm P})/f(W_1^{\mathrm S})$;
panel (b) reports $T_{\mathrm P}/T_{\mathrm S}$, where
$T_{\mathrm P}$ and $T_{\mathrm S}$ are the numbers of iterations required
by practical and standard NS, respectively, to reduce the objective
to $10^{-6}$ of its initial value.
}
\vspace{0.25em}
\label{tab:ns-one-step}
\small
\setlength{\tabcolsep}{10pt}
\setlength{\arrayrulewidth}{0.5pt}
\begin{tabular}{r|rrrrr}
\multicolumn{6}{c}{ \textbf{(a) One-step contraction}} \\[3pt]
\hline
$h\backslash\kappa$
& $100$ & $10^3$ & $10^4$ & $10^5$ & $10^6$ \\
\hline
$1$   & $4.9749\times10^{-5}$ & $0.0762$ & $0.8655$ & $0.9865$ & $0.9987$ \\
$10$  & $5.4271\times10^{-7}$ & $0.7462$ & $0.9744$ & $0.9974$ & $0.9997$ \\
$100$ & $0.0839$ & $0.8685$ & $0.9868$ & $0.9987$ & $0.9999$ \\
\hline
\noalign{\vskip 8pt}
\multicolumn{6}{c}{\textbf{(b) Full trajectory (numerical simulations)}} \\[3pt]
\hline
$h\backslash\kappa$
& $100$ & $10^3$ & $10^4$ & $10^5$ & $10^6$ \\
\hline
$1$   & $0.1250$ & $0.0376$ & $0.0176$ & $0.0112$ & $0.0280$ \\
$10$  & $0.0208$ & $0.0440$ & $0.0033$ & $0.0012$ & $0.0011$ \\
$100$ & $0.0298$ & $0.0157$ & $0.0175$ & $0.0263$ & $0.0288$ \\
\hline
\end{tabular}
\end{table}
\end{proposition}

\section{Experiments} \label{sec:simu}

As full gradient evaluation is impractical for large language models,
we examine Muon's $\mathcal{O}(T^{-\hspace{-0.1em}1/4})$ convergence rate
using ResNet-32 \citep{He_2016_CVPR} with 118M parameters (Figure~\ref{fig:resnet_speed}, left), where its
Frobenius gradient norm follows an overall trend qualitatively similar
to the $T^{-1/4}$ reference curve.
We further examine Muon's theoretical $\sqrt{d}$ improvement over AdamW
in the Frobenius norm using ResNet-32 of varying dimensions
(Figure~\ref{fig:resnet_speed}, right), 
where the normalized
Muon-to-AdamW gradient norm ratios, denoted by
$\sqrt{d}\|\nabla f(W_{\mathrm{Muon}})\|_F/
\|\nabla f(W_{\mathrm{AdamW}})\|_F$,
remain broadly stable across dimensions despite fluctuations, suggesting a growing relative advantage
with dimension.

We also show on a 124M-parameter GPT-2-style model \citep{radford2019language} that a large factor
$(1\hspace{-0.1em}-\hspace{-0.1em}\epsilon_5)^{-\hspace{-0.1em}1}$
arising from the worst-case NS orthogonalization error 
does not necessarily  imply
slow convergence (Figure~\ref{fig:124loss}, left),
even when it is on the order of $10^7$ for standard NS and $10^5$
for practical NS (Figure~\ref{fig:124loss}, right).
See Appendices~\ref{app:convergence} and~\ref{app:error} for further details,
including the setup, larger-model results, and hyperparameter sweeps.




\begin{figure}[h]
    \centering 
    \includegraphics[width=1\linewidth]{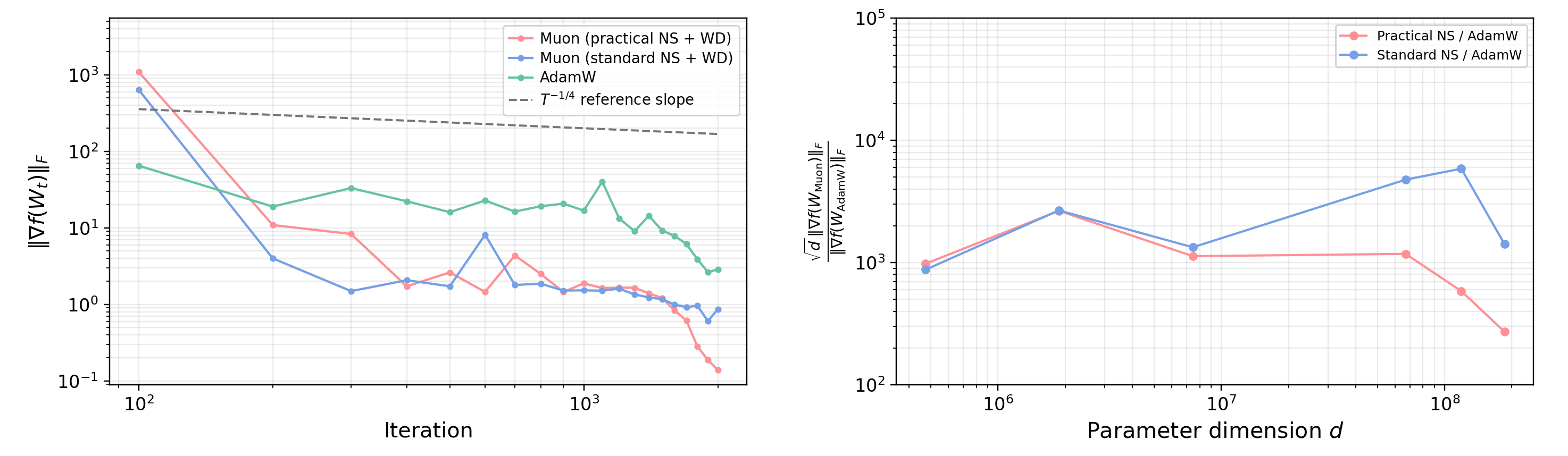}
\caption{Empirical $\mathcal{O}(T^{-1/4})$ convergence and $\sqrt{d}$ improvement over AdamW on ResNet-32.}    \label{fig:resnet_speed}
\end{figure}

\begin{figure}[h]
    \centering
    \includegraphics[width=1\linewidth]{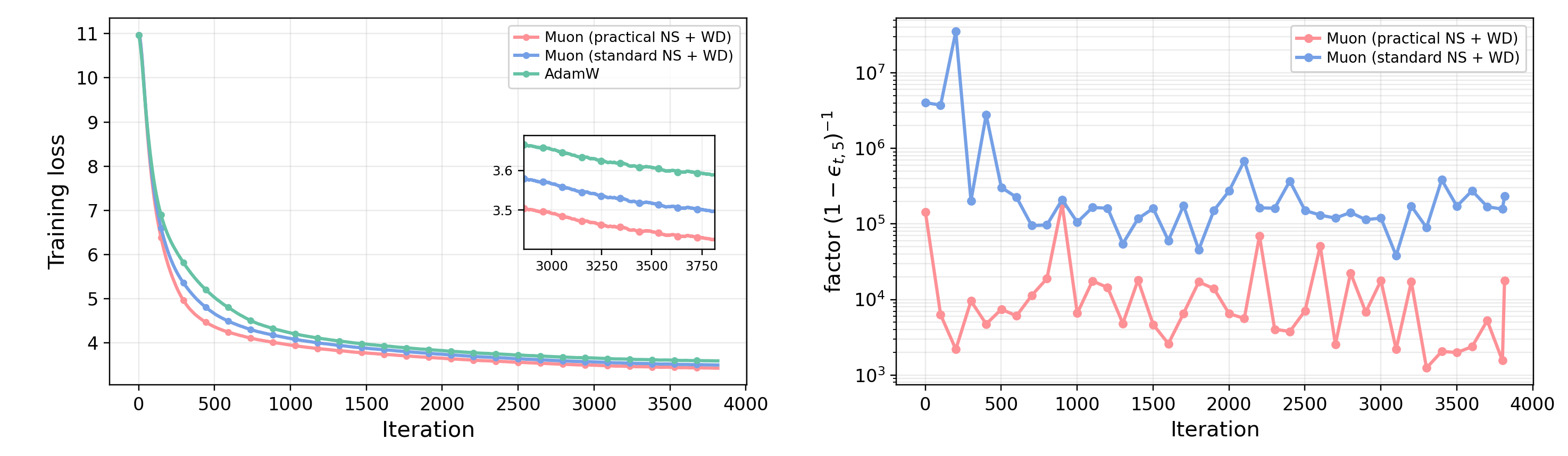}
    \caption{Training loss and orthogonalization error factor
    $(\hspace{-0.1em}1\hspace{-0.09em}-\hspace{-0.05em}\epsilon_{t\hspace{-0.05em},\hspace{-0.05em}5}\hspace{-0.1em})^{\hspace{-0.1em}-\hspace{-0.1em}1}$ on\hspace{0.17em}the\hspace{0.14em}124M\hspace{0.2em}GPT-2-style\hspace{0.17em}model.}    
    \label{fig:124loss}
\end{figure}


\section{Conclusion} \label{sec:con}
In this paper, we provided an optimization interpretation and established convergence 
for Muon with practical Newton--Schulz iterations and decoupled weight decay.
Our interpretation reveals a dynamic regularization effect that vanishes as stationarity is approached, preserving the original objective as the convergence target.
On this basis, we established an $\mathcal{O}(T^{-1/4})$ convergence rate in the stochastic nonconvex setting, 
with a $\sqrt{d}$ improvement in dimension dependence over the best known AdamW's convergence rate in terms of the expected Frobenius norm of the gradient.
Future work of particular interest includes understanding Muon's advantages and improving its performance.

\subsection*{AI use statement}


In this work, we used generative AI tools to improve the grammar of the manuscript.
We have not used generative AI tools for research ideation,
theoretical analysis, or the interpretation of experimental results.
We have reviewed all AI-assisted work to ensure that it preserves
the intended meaning and technical content.
We take responsibility for the final content of this work,
including text, claims or artifacts produced with the aid of generative AI.

\subsubsection*{Acknowledgments}
This work was supported in part by the National Natural Science
Foundation of China under Grants 62503365, 62573319, 62271424,
and 62088101.

\bibliography{iclr2027_conference}
\bibliographystyle{iclr2027_conference}

\appendix

\section{Additional related work} \label{app:realted}

\paragraph{Convergence of AdamW.}

Although the convergence of Adam and its variants has been extensively
studied \citep{reddi2018convergence,zhang2022adam}, decoupled WD complicates both
the optimization interpretation and convergence analysis of AdamW.
\citet{xie2024implicit} show that, if deterministic AdamW converges,
its limit is a KKT point of an $\ell_\infty$-constrained problem.
\citet{zhou2024adamw} provide the first convergence analysis of AdamW through 
a surrogate objective combining the original loss 
with a dynamic weighted $\ell_2$ regularizer.
However, their analysis requires an exponentially decreasing of the WD
coefficient, progressively weakening regularization.
More recently, \citet{li2026frac} establish the best known convergence rate
\eqref{rate:adamW} for stationarity of the original objective
with a WD coefficient held fixed across iterations.
Inspired by these works, we investigate the role of WD in Muon 
while additionally accounting for its interaction with practical NS, 
thereby providing an optimization interpretation and convergence analysis for practical Muon.

\paragraph{Newton--Schulz iterations.}
NS iterations approximately orthogonalize matrices through
repeated matrix polynomial updates
\citep{bjorck1971iterative,higham2008functions}.
Standard Taylor-based polynomials are designed to progressively improve the orthogonalization accuracy, 
whereas Muon's practical polynomial is tuned to move more rapidly toward orthogonalization within a small fixed number of iterations
while allowing deviations from exact orthogonalization \citep{jordan2024muon}.
Recent methods such as Polar Express \citep{amsel2026polar} and CANS \citep{grishina2025accelerating} accelerate orthogonalization using a sequence of NS polynomials, each optimized for the singular-value range produced by the previous iteration.

\paragraph{Optimization interpretations of Muon.}
Muon's update direction is often interpreted as a steepest
descent direction for a surrogate objective under
a suitable norm constraint.
Examples include spectral-norm steepest descent
\citep{bernstein2024old} and spectral-ball LMO formulations
\citep{pethick2025training} for Muon with exact orthogonalization,
as well as Frank--Wolfe interpretations
\citep{pethick2025training,sfyraki2026lions}
and the Lion-$\mathcal{K}$ framework \citep{chen2026muon},
which additionally account for WD.
Beyond explaining the update direction, the Frank--Wolfe
and Lion-$\mathcal{K}$ interpretations also recast Muon with
exact orthogonalization and WD as a method for minimizing
the original objective under a spectral-norm constraint
on the parameters.
A complementary perspective studies Muon through matrix
preconditioning. \citet{bernstein2024old} interpret Muon's
exact orthogonalization as Shampoo preconditioning without
accumulation, while \citet{lau2025polargrad,ma2026preconditioning}
interpret orthogonalization as preconditioning with inverse
square roots of gradient Gram matrices, which can offer benefits such as reduced gradient anisotropy
and faster convergence on certain structured matrix problems.
Motivated by these perspectives, we also interpret Muon's
update direction as a descent direction for an underlying
objective and provide a corresponding preconditioning
interpretation.

\paragraph{Understanding Muon's advantages over Adam and AdamW.}
Beyond the success discussed in the introduction,
existing work investigates the mechanisms underlying Muon's
advantages over Adam and AdamW
from the perspectives of long-tail learning, directional
curvature, and convergence complexity.
For long-tail learning, these advantages are linked to
Muon's more uniform update singular values, which can
reduce the dominance of frequent patterns and promote
more balanced learning across frequencies
\citep{wang2026tail,vasudeva2026muon,wang2026associative}.
Regarding curvature, the preconditioning interpretation discussed above
provides one theoretical perspective.
\citet{su2025isotropic} offers another through an isotropic curvature
model for a single optimization step, in which gradient orthogonalization
is optimal when the curvature penalty has a sufficiently sharp kink.
Empirical studies complement these theoretical perspectives by showing
that Muon exhibits lower normalized directional curvature than
Adam at matched validation loss \citep{wang2026curvature} and
allocates a larger fraction of its update to low-curvature
Hessian subspaces than AdamW \citep{kim2026amuse}.
In terms of convergence complexity,
\citet{ma2026preconditioning} and \citet{paquette2026phases}
compare Muon with exact orthogonalization against SignGD and SignSGD, respectively,
while \citet{kim2026muon} compare Muon with standard NS
against SGD with momentum, obtaining sharper rank dependence
in convergence rates for nuclear norm stationarity.
We complement these perspectives by directly comparing the convergence rate we establish for practical Muon in Theorem~\ref{thm:main} 
with the best known AdamW's rate \eqref{rate:adamW}, showing Muon's sharper dimension dependence when stationarity is measured in the Frobenius norm.

\section{Complete proofs} \label{app:proof}

This appendix provides the complete proofs of Lemma~\ref{lem:ns},
Lemma~\ref{lem:tracking}, and Proposition~\ref{prop:two-level}.
We begin with Lemma~\ref{lem:ns}, deriving the sharp constants omitted from the main text.

\begin{lemnsrestated}[NS direction alignment]
For both practical and standard NS,
the momentum $M_t$ and the resulting
direction $D_t$ in Algorithm~\ref{alg:muon} satisfy
\begin{align}
        \langle M_t,D_t\rangle_F
    \ge C_2\nf{M_t}-C_2\varepsilon/4,
    \qquad
    \nf{D_t}\le C_3, \nonumber
\end{align}
where $C_2\hspace{-0.1em}=\hspace{-0.1em}\min_{s\in[0,1]}\chi_5(s)\hspace{-0.1em}>\hspace{-0.1em}0,~C_1\hspace{-0.1em}=\hspace{-0.1em}\max_{s\in[0,1]}\chi_5(s)\hspace{-0.1em}>\hspace{-0.1em}0$,
as defined above, and
\begin{align}
C_3=\min\{C_1,
q_{\max}\sqrt{r}\},\qquad
q_{\max}=\max\nolimits_{s\in[0,1]}q_5(s),\qquad
r=\min\{m,n\}. \nonumber
\end{align}
 In particular, 
 \vspace{-0.5em}
\[
    (C_2,C_1,q_{\max})=
    \begin{cases}
    (0.6964\ldots,\,484.8762\ldots,\,1.2023\ldots),
        & \text{practical NS},\\
    (1,\,a^5,\,1),
        & \text{standard NS}.
    \end{cases}
\]
\end{lemnsrestated}

\begin{proof}
{\bf (i)}
The proof of Proposition~\ref{prop:dynamic-regularization}
ensures that $\chi_5$, $C_1$, $C_2$, and $q_{\max}$ are
well defined. Their definitions directly give
\begin{align*}
    C_2s\le q_5(s)\le C_1s,
    \qquad
    0\le q_5(s)\le q_{\max},
    \qquad s\in[0,1].
\end{align*}
We first establish the bounds in Lemma~\ref{lem:ns} and then
determine the sharp constants $C_1$, $C_2$, and $q_{\max}$ for both practical and standard NS.
By Lemma~\ref{lem:ns-polynomials},
$D_t=U\operatorname{diag}(q_5(s_{i,t}),\ldots,q_5(s_{r_t,t})V^\top$.
Then \eqref{eq:both} yields
\[
 \langle M_t,D_t \rangle_F
 =\tau_t\sum\nolimits_i s_{i,t}q_5(s_{i,t})
 \ge C_2\tau_t\sum\nolimits_i s_{i,t}^2= C_2\nf{M_t}^2/\tau_t,
\]
where
 $\nf{M_t}^2/\tau_t\ge \nf{M_t}-\varepsilon/4$, since for $\nf{M_t}\le\varepsilon$ the inequality reduces to
$(\nf{M_t}-\varepsilon/2)^2/\varepsilon\ge0$, whereas for $\nf{M_t}\ge\varepsilon$, the left-hand side equals $\nf{M_t}$.

Moreover, \eqref{eq:both} together with $\sum_i s_{i,t}^2=\nf{M_t}^2/\tau_t^2\le1$ gives
\[
 \nf{D_t}=\sqrt{\sum\nolimits_iq_5(s_{i,t})^2}
 \le\min\Big\{C_1\sqrt{\sum\nolimits_i s_{i,t}^2},
 q_{\max}\sqrt{r}\Big\}
 \le C_3.
\]

{\bf (ii)}
We then determine the sharp constants $C_1$, $C_2$, and $q_{\max}$, which are illustrated in Figure~\ref{sharp:constant}.
For practical and standard NS, respectively, we denote $\chi_5$
by $\chi_{\mathrm P,5}$ and $\chi_{\mathrm S,5}^{(k)}$, and
$q_{\max}$ by $q_{\mathrm P,\max}$ and $q_{\mathrm S,\max}^{(k)}$.
We first consider the practical NS and establish the following
bounds with both constants being sharp:
\begin{align}
    q_{\mathrm P,5}(1)s
    \le q_{\mathrm P,5}(s)
    \le a_{\mathrm P}^5s,
    \qquad \forall s\in[0,1]. \nonumber
\end{align}
Hence $C_2=q_{\mathrm P,5}(1)=0.6964\ldots$ and $C_1=a_{\mathrm P}^5=484.8762\ldots$.

\begin{figure}[H]
    \centering
    \includegraphics[width=1\linewidth]{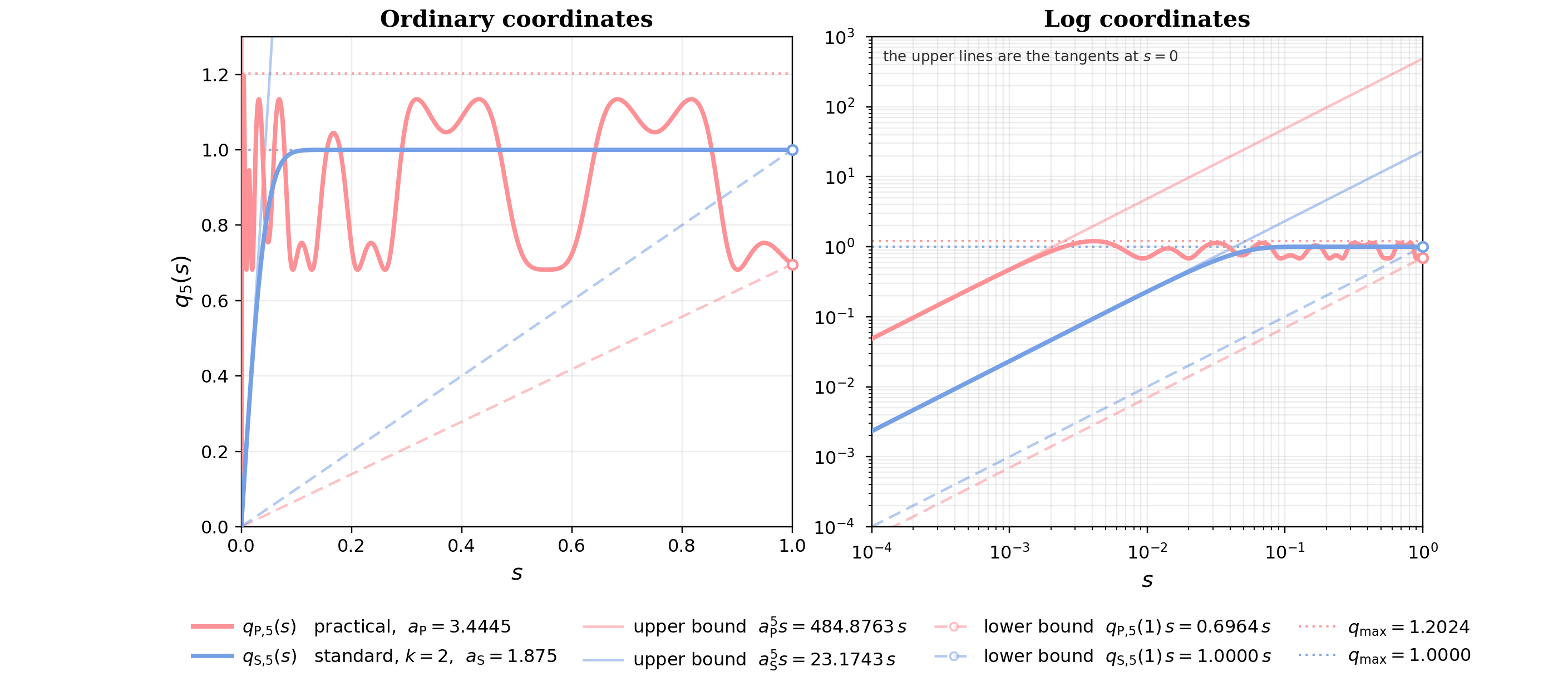}
    \caption{
        Sharp bounds for five-step Newton--Schulz maps.
    }
    \label{sharp:constant}
\end{figure}
The two positive roots of
$q_{\mathrm P}'(s)=a_{\mathrm P}+3b_{\mathrm P}s^2
+5c_{\mathrm P}s^4$ are
\[
s_-\hspace{-0.1em}=\hspace{-0.1em}
\sqrt{\frac{-3b_{\mathrm P}-\sqrt{9b_{\mathrm P}^2-20a_{\mathrm P}c_{\mathrm P}}}
{10c_{\mathrm P}}}\hspace{-0.1em}=\hspace{-0.1em}
0.5545\ldots,
~s_+\hspace{-0.1em}=\hspace{-0.1em}
\sqrt{\frac{-3b_{\mathrm P}+\sqrt{9b_{\mathrm P}^2-20a_{\mathrm P}c_{\mathrm P}}}
{10c_{\mathrm P}}}
\hspace{-0.1em}=\hspace{-0.1em}1.0501\ldots.
\]
We have $\bar{q}_{\mathrm P}=\max_{s\in[0,1]}q_{\mathrm P}(s)=q_{\mathrm P}(s_-)=1.2023\ldots$, and 
denote $\underline q_{\mathrm P}=q_{\mathrm P}(s_+)=0.6818\ldots$.
The monotonicity intervals of $q_{\mathrm P}$, together with
$q_{\mathrm P}(\bar{q}_{\mathrm P})=0.9465\ldots<\bar{q}_{\mathrm P}$, gives
\[
 q_{\mathrm P}([0,\bar{q}_{\mathrm P}])\subset[0,\bar{q}_{\mathrm P}],
 \qquad
 q_{\mathrm P}([\underline q_{\mathrm P},\bar{q}_{\mathrm P}])
 \subset[\underline q_{\mathrm P},\bar{q}_{\mathrm P}].
\]
Since $q_{\mathrm P}$ is continuous, $q_{\mathrm P}(0)=0$,
and $q_{\mathrm P}(s_-)=\bar q_{\mathrm P}$ with
$s_-\in[0,1]\subset[0,\bar q_{\mathrm P}]$,
the intermediate value theorem \citep{rudin2021principles} and the invariance of
$[0,\bar q_{\mathrm P}]$ give
\[
    q_{\mathrm P}([0,1])
    =q_{\mathrm P}([0,\bar q_{\mathrm P}])
    =[0,\bar q_{\mathrm P}].
\]
By induction,
$q_{\mathrm P,j}([0,1])=[0,\bar q_{\mathrm P}]$
for $j=1,\ldots,5$. Therefore,
\[
    q_{\mathrm{P},\max}
    =\max\nolimits_{s\in[0,1]}q_{\mathrm P,5}(s)
    =\bar q_{\mathrm P}=1.2023\ldots.
\]
Recall that $p_{\mathrm P}(s)=a_{\mathrm P}+b_{\mathrm P}s+c_{\mathrm P}s^2$ and 
$q_{\mathrm P}(s)=sp_{\mathrm P}(s^2)$.
A direct calculation shows that $p_{\mathrm P}$ is decreasing on $[0,\underline q_{\mathrm P}^2]$ and
$
 p_{\mathrm P}(\underline q_{\mathrm P}^2)
 =1.6636\ldots>1,
$
so $q_{\mathrm P}(s)\ge s$ for $s\in[0,\underline q_{\mathrm P}]$.
Consequently, considering the process $q_{\mathrm P,0}(s)=s$ and $q_{\mathrm P,j+1}(s)=q_{\mathrm P}(q_{\mathrm P,j}(s))=q_{\mathrm P,j}(s)p(q_{\mathrm P,j}(s)^2)$ for $j=0,\ldots,4$,
every orbit starting from $s\in[0,1]$ either remains below
$\underline q_{\mathrm P}$ and is
nondecreasing, or enters the invariant interval
$[\underline q_{\mathrm P},\bar q_{\mathrm P}]$. Hence
\begin{equation}
 q_{\mathrm P,5}(s)\ge\min\{s,\underline q_{\mathrm P}\},
 \qquad 0\le s\le1.      \nonumber
\end{equation}

Set
$
 s_\star=\underline q_{\mathrm P}/q_{\mathrm P,5}(1)
 =0.9790\ldots .
$
For $s\in [0, s_\star]$,  
$
 q_{\mathrm P,5}(s)\ge\min\{s,\underline q_{\mathrm P}\}\ge q_{\mathrm P,5}(1)s.$
For $s\in[s_\star,1]$, we first analyse the monotonicity of
$q_{\mathrm P,5}$.
The following array tracks
$q_{\mathrm P,j}(s)$ after $j$ NS steps, with
$q_{\mathrm P,0}(s)=s$.
\[
\begin{array}{c|c|c|c}
\text{NS steps }j
& \text{Iterate}
& \text{Enclosing interval }I_j
& I_j\subset \\ \hline
0 & s
  & [0.9790,\,1.0000] & (s_-,s_+)\\
1 & q_{\mathrm P,1}(s)
  & [0.7010,\,0.7187] & (s_-,s_+)\\
2 & q_{\mathrm P,2}(s)
  & [1.0924,\,1.1137] & (s_+,+\infty)\\
3 & q_{\mathrm P,3}(s)
  & [0.6983,\,0.7209] & (s_-,s_+)\\
4 & q_{\mathrm P,4}(s)
  & [1.0897,\,1.1167] & (s_+,+\infty)
\end{array}
\]
These enclosures follow successively from
$[s_\star,1]\subseteq I_0$ and
$q_{\mathrm P}(I_j)\subseteq I_{j+1}$ for $j=0,\ldots,3$,
verified using the endpoint values and the monotonicity
of $q_{\mathrm P}$ on each $I_j$.
We know that
$q_{\mathrm P}'(q_{\mathrm P,j}(s))<0$ for $j=0,1,3$
and $q_{\mathrm P}'(q_{\mathrm P,j}(s))>0$ for $j=2,4$.
Thus the chain-rule product contains three negative factors
and two positive factors, and
$
 q_{\mathrm P,5}'(s)
 =\prod\nolimits_{j=0}^4 q_{\mathrm P}'(q_{\mathrm P,j}(s))<0.
$
Therefore, for $s\in [0, s_\star]$,  
$
 q_{\mathrm P,5}(s)
 \ge q_{\mathrm P,5}(1)
 \ge q_{\mathrm P,5}(1)s.
$
Together, we obtain
\begin{equation}
 q_{\mathrm P,5}(1)s\le q_{\mathrm P,5}(s),\qquad 0\le s\le1. \nonumber
\end{equation}
Since equality holds at $x=1$, $C_2=q_{\mathrm P,5}(1)$ is sharp.

Note that $\bar q_{\mathrm P}^2<-b_{\mathrm P}/c_{\mathrm P}$, invariance of
$[0,\bar q_{\mathrm P}]$ gives
$p_{\mathrm P}(q_{\mathrm P,j}(s)^2)\le p_{\mathrm P}(0)=a_{\mathrm P}$. Thus
\begin{align}
 q_{\mathrm P,5}(s)
 &=s\prod\nolimits_{j=0}^4
 p_{\mathrm P}(q_{\mathrm P,j}(s)^2)
 \le a_{\mathrm P}^5s, \nonumber
\end{align}
and
$C_1=\max_{s\in[0,1]}\chi_{\mathrm P,5}(s)=\chi_{\mathrm P,5}(0)=a_{\mathrm P}^5=484.8762\ldots$.

{\bf (iii)} We next consider the standard NS.
The monotonicity of $q_{\mathrm S}^{(k)}$, together with
$q_{\mathrm S}^{(k)}(0)=0$, and $q_{\mathrm S}^{(k)}(1)=1$  in Lemma~\ref{lem:ns-polynomials}  gives
$q_{\mathrm S}^{(k)}([0,1])=[0,1]$.
By induction, 
$q_{\mathrm S,j}^{(k)}([0,1])=[0,1]$ and
$q_{\mathrm S,\max}^{(k)}=1$.
Moreover, the definition $p_{\mathrm S}^{(k)}(s)
    =1+\sum_{\ell=1}^{k}
    \frac{(2\ell)!}{4^\ell(\ell!)^2}(1-s)^\ell$ yields
\[
    1\le p_{\mathrm S}(s)\le p_{\mathrm S}(0)=a,
    \qquad s\in[0,1].
\]
Thus,
\begin{align}    
    s\le q_{\mathrm S,5}(s)
    =s\prod\nolimits_{j=0}^{4}
    p_{\mathrm S}\bigl(q_{\mathrm S,j}(s)^2\bigr)\le\min\{a^5s,1\},
    \qquad s\in[0,1], \label{both:standard}
\end{align}
which implies $C_2=1$ since $q_{\mathrm S,5}(1)=1$, and $C_1=\max_{s\in[0,1]}\chi_{\mathrm S,5}^{(k)}(s)=\chi_{\mathrm S,5}^{(k)}(0)=a^5$.
\end{proof}

We next prove Lemma~\ref{lem:tracking}, which bounds the tracking error of the Nesterov-style momentum.

\begin{lemtrackingrestated}[Momentum tracking]
Under the assumptions and parameter choices of Theorem~\ref{thm:main},
\begin{align}
 \frac1T\hspace{-0.1em}\sum\nolimits_{t=0}^{T-1}\hspace{-0.1em}\E\nf{M_t\hspace{-0.1em}-\hspace{-0.1em}\nabla f(W_t)}
 \le{}&\frac{\nf{B_{-1}\hspace{-0.1em}-\hspace{-0.1em}\nabla f(W_0)}}{c_2}T^{-1/2}
 \hspace{-0.1em}+\hspace{-0.1em}\Big(
 \frac{Lc_1 C_4}{c_2}
 \hspace{-0.1em}+\hspace{-0.1em}\sqrt{2c_2}\,\frac{\sigma}{\sqrt{b_{\mathrm{mb}}}}
 \Big)T^{-1/4}\hspace{-0.1em}.                    \nonumber                        
\end{align}                           
\end{lemtrackingrestated}

\begin{proof}
Unrolling the two momentum recursions defining $B_t$ and $M_t$ in
Algorithm~\ref{alg:muon} gives
\begin{align}
 B_t&=\beta^{t+1}B_{-1}
 +(1-\beta)\sum\nolimits_{j=0}^t\beta^{t-j}G_j, \notag\\
 M_t&=\beta^{t+2}B_{-1}
 +(1-\beta)\sum\nolimits_{j=0}^{t-1}\beta^{t-j+1}G_j
 +(1-\beta^2)G_t. 
                                          \label{eq:momentum-expanded}
\end{align}
Subtracting $\nabla f(W_t)$ from \eqref{eq:momentum-expanded}, 
using $G_j=\nabla f(W_j)+G_j-\nabla f(W_j)$, and the reindexing identity
\begin{align*}
&(1-\beta)\sum_{j=0}^{t-1}\nolimits\beta^{t-j+1}\nabla  f(W_j)-\beta^2\nabla  f(W_t)  \\
&=
\sum_{j=0}^{t-1}\nolimits\beta^{t-j+1}\nabla  f(W_j)
-\sum_{j=0}^{t-1}\nolimits\beta^{t-j+2}\nabla  f(W_j)
-\beta^2\nabla  f(W_t) \\
&=
\sum_{j=1}^{t}\nolimits\beta^{t-j+2}\nabla f(W_{j-1})
-\beta^{t+2}\nabla f(W_0)
-\sum_{j=1}^{t}\nolimits\beta^{t-j+2}\nabla f(W_j)  \\
&=
-\beta^{t+2}\nabla f(W_0)
-\sum_{j=1}^{t}\nolimits\beta^{t-j+2}
\bigl(\nabla f(W_j)-\nabla f(W_{j-1})\bigr),
\end{align*}
 we obtain
\begin{align}
 e_t={}&\beta^{t+2}\bigl(B_{-1}-\nabla f(W_0)\bigr)  -\sum\nolimits_{j=1}^t\beta^{t-j+2}
 \bigl(\nabla f(W_j)-\nabla f(W_{j-1})\bigr)+N_t,
                                                        \label{eq:tracking-identity}\\
 N_t={}&(1-\beta^2)\bigl(G_t-\nabla f(W_t)\bigr)  
 +(1-\beta)\sum_{j=0}^{t-1}\nolimits\beta^{t-j+1}
 \bigl(G_j-\nabla f(W_j)\bigr).                          \notag
\end{align}
Assumption~\ref{ass:smooth} and \eqref{eq:direction-control} imply
$\nf{\nabla f(W_j)-\nabla f(W_{j-1})}\le L\eta C_4$.
Thus the triangle inequality and the
geometric-series bounds in \eqref{eq:tracking-identity} give
\begin{equation}
 \frac1T\sum_{t=0}^{T-1}\nolimits\E\nf{e_t}
 \le\frac{\nf{B_{-1}-\nabla f(W_0)}}{T(1-\beta)}
 +\frac{L\eta C_4}{1-\beta}
 +\frac1T\sum_{t=0}^{T-1}\nolimits\E\nf{N_t}.                     \label{eq:tracking-triangle}
\end{equation}
By independent sampling and Assumption~\ref{ass:noise}, we have
\begin{align*}
 \E\hspace{-0.1em}\nf{N_t}^2
 \hspace{-0.1em}\le\hspace{-0.1em}\frac{\sigma^2}{b_{\mathrm{mb}}}
 \Big[(1\hspace{-0.1em}-\hspace{-0.1em}\beta^2)^2\hspace{-0.1em}+\hspace{-0.1em}
 (1\hspace{-0.1em}-\hspace{-0.1em}\beta)^2\sum_{j=0}^{t-1}\beta^{2(t-j+1)}\Big] 
 \hspace{-0.2em}\le\hspace{-0.1em}\frac{\sigma^2}{b_{\mathrm{mb}}}
 \frac{(1-\beta)(1+2\beta-2\beta^3)}{1+\beta}
 \hspace{-0.1em}\le\hspace{-0.1em}\frac{2(1\hspace{-0.1em}-\hspace{-0.1em}\beta)\sigma^2}{b_{\mathrm{mb}}}\hspace{-0.15em}.
\end{align*}
By the Cauchy--Schwarz inequality,
\[
\frac{1}{T}\sum_{t=0}^{T-1}\nolimits\E\nf{N_t}
\le
\Big(
\frac{1}{T}\sum_{t=0}^{T-1}\nolimits\E\nf{N_t}^2
\Big)^{1/2}
\le \frac{\sigma}{\sqrt{b_{\mathrm{mb}}}}
\sqrt{2(1-\beta)}.
\]
Substituting this bound, 
$\eta=c_1T^{-3/4}$, and $(1-\beta)=c_2T^{-1/2}$ into \eqref{eq:tracking-triangle} yields
the desired result.
\end{proof}




Finally, we prove Proposition~\ref{prop:two-level}, which illustrates the
optimization advantage of practical NS over standard NS on the ill-conditioned
quadratic.

\begin{repproposition}[Ill-conditioned quadratic]
Let $h,\ell,\kappa\in\mathbb N$, with $\kappa\gg1$, $\ell=h\kappa$, and $r=h+\ell$.
Define
\[
Q=\operatorname{diag}(\kappa I_h,I_\ell),\qquad
f(W)=\frac12\operatorname{tr}(W^\top QW),\qquad
W_0=I_r.
\]
The high-curvature block accounts for most of the gradient energy,
while each block contributes half of the objective value.
We consider practical NS or standard NS with $k=2$, and normalize
without an $\varepsilon$-stabilizer, stopping when the gradient vanishes.

\textbf{(i) One-step contraction.} Denote 
$
\tau_0\hspace{-0.1em}=\hspace{-0.1em}\nf{\nabla f(W_0)\hspace{-0.1em}}=\hspace{-0.1em}\sqrt{h\kappa(\kappa+1)},~
u=q_5(\kappa/\tau_0),~v=q_5(1/\tau_0).
$
For deterministic Muon without momentum or WD, the respective optimal stepsizes minimizing $f(W_1)$ for the two NS choices yield
\begin{equation}
\frac{f(W_1)}{f(W_0)}
=\frac{h\kappa\ell(u-v)^2}
{(h\kappa+\ell)(h\kappa u^2+\ell v^2)}=
\frac{(u-v)^2}{2(u^2+v^2)}
\le
\left(\frac{u-v}{u+v}\right)^2.  \nonumber
\end{equation}

Consider the example with $h=1$ and $\ell=\kappa=100$, five practical and standard NS
iterations give $(u,v)\approx(0.7020,0.6968)$ and
$(1,0.2274)$, respectively, leaving $0.0014\%$ and $28.3789\%$
of $f(W_0)$ after the one-step update.

For fixed $h$ and $\kappa\to\infty$, substituting
$u=q_5(1/\sqrt h)+\mathcal O(\kappa^{-1})$ and
$v=a^5/(\kappa\sqrt h)+\mathcal O(\kappa^{-2})$ gives
\[
\frac{f(W_1)}{f(W_0)}
=\frac12-\frac{a^5}{\kappa\sqrt h\,q_5(1/\sqrt h)}
+\mathcal O(\kappa^{-2}).
\]
When $q_5(1/\sqrt h)$ is close to one for both NS
choices, the leading-order improvement is mainly governed by $a^5$,
which is approximately $484.8762$ for practical NS and $23.1743$
for standard NS. Thus, practical NS yields a substantially
stronger one-step contraction.

\textbf{(ii) Full optimization trajectory.}
Consider the example with $h=1$ and $\ell=\kappa=100$.
Let $W_t^{\mathrm P}$ and $W_t^{\mathrm S}$ denote the Muon iterates generated by practical and standard NS, respectively, 
using the optimal stepsizes at each iteration $t$ that minimize $f(W_{t+1}^{\mathrm P})$ and $f(W_{t+1}^{\mathrm S})$.
Then
\begin{equation}
\frac{f(W_t^{\mathrm P})}{f(W_0)}\le (6.4514\times10^{-5})^t,
\qquad
\frac{f(W_t^{\mathrm S})}{f(W_0)}\ge(0.1406)^t. \nonumber
\end{equation}
Further comparisons for different values of $h$ and $\kappa$
are provided in Table~\ref{tab:ns-one-step}.
\end{repproposition}



\begin{proof}
\noindent\textbf{(i) One-step contraction.}
The normalized gradient and the NS direction are
\[
\frac{\nabla f(W_0)}{\tau_0}
=\operatorname{diag}\left(
\frac{\kappa}{\tau_0}I_h,\frac1{\tau_0}I_\ell
\right),
\qquad
D_0=\operatorname{diag}(uI_h,vI_\ell).
\]
Hence,
\[
f(W_0-\eta D_0)
=\frac12\big(
h\kappa(1-\eta u)^2+\ell(1-\eta v)^2
\big),
\]
and the optimal stepsize minimizing $f(W_0-\eta D_0)$ is
\[
\eta_0^\star
=\frac{h\kappa u+\ell v}{h\kappa u^2+\ell v^2}.
\]
With \(\eta=\eta_0^\star\), we have
\begin{align*}
\frac{f(W_1)}{f(W_0)}
&=1-\frac{(h\kappa u+\ell v)^2}
{(h\kappa+\ell)(h\kappa u^2+\ell v^2)}
=\frac{h\kappa\ell(u-v)^2}
{(h\kappa+\ell)(h\kappa u^2+\ell v^2)}
=\frac{(u-v)^2}{2(u^2+v^2)}
\le\left(\frac{u-v}{u+v}\right)^2,
\end{align*}
where the last two steps use $\ell=h\kappa$ and
$(u+v)^2\le2(u^2+v^2)$.

For $h=1$ and $\ell=\kappa=100$, direct evaluation of the
five NS iterations gives
\[
(u_{\mathrm P},v_{\mathrm P})
\approx(0.7020\ldots,0.6967\ldots),
\qquad
(u_{\mathrm S},v_{\mathrm S})
\approx(1,0.2273\ldots).
\]
Hence, the values of $f(W_1)/f(W_0)$ are approximately
$1.4118\times10^{-5}$ and $0.2837\ldots$ for practical and standard NS,
respectively, which gives the stated percentages.

Next, fix $h$ and let $\kappa\to\infty$ with
$\ell=h\kappa\in\mathbb N$. We have
\[
\frac{\kappa}{\tau_0}
=\frac1{\sqrt h}(1+\kappa^{-1})^{-1/2}
=\frac1{\sqrt h}+\mathcal O(\kappa^{-1}),
\qquad
\frac1{\tau_0}=\frac1{\kappa\sqrt h}(1+\kappa^{-1})^{-1/2}
=\frac1{\kappa\sqrt h}+\mathcal O(\kappa^{-2}).
\]
Since $q(s)=as+bs^3+cs^5$, its fivefold composition 
satisfies $q_5'(0)=a^5$. Consequently,
\[
u=q_5(1/\sqrt h)+\mathcal O(\kappa^{-1}),
\qquad
v=\frac{a^5}{\kappa\sqrt h}+\mathcal O(\kappa^{-2}),
\]
and, because 
$u>0$,
\[
\frac vu
=\frac{a^5}{\kappa\sqrt h\,q_5(1/\sqrt h)}
+\mathcal O(\kappa^{-2}).
\]
Using $(1-z)/(1+z)=1-2z+\mathcal O(z^2)$ as $z\to0$ yields
\[
\frac{f(W_1)}{f(W_0)}
\le
\left(
1-\frac{2a^5}{\kappa\sqrt h\,q_5(1/\sqrt h)}
\right)^2
+\mathcal O(\kappa^{-2}).
\]
In fact, the exact contraction formula also gives
\[
\frac{f(W_1)}{f(W_0)}
=\frac12-\frac{v/u}{1+(v/u)^2}
=\frac12-
\frac{a^5}{\kappa\sqrt h\,q_5(1/\sqrt h)}
+\mathcal O(\kappa^{-2}).
\]
When $q_5(1/\sqrt h)$ is close to one,
 the leading-order comparison is governed by
$a^5$, with
$a_{\mathrm P}^5\approx484.8762$ and
$a_{\mathrm S}^5\approx23.1743$.

\medskip
\noindent\textbf{(ii) Full optimization trajectory.}
Set $h=1$ and $\ell=\kappa=100$.
The block structure is preserved throughout the iterations, so write
\[
W_t=\operatorname{diag}(x_t,y_tI_{100}),
\qquad x_0=y_0=1.
\]
Then
\[
f(W_t)=50(x_t^2+y_t^2),
\qquad
\nabla f(W_t)=\operatorname{diag}(100x_t,y_tI_{100}).
\]
Whenever $x_t,y_t\ne0$, define
\[
z_t=\frac{|x_t|}{|y_t|},
\quad
u(z)=q_5\left(\frac{100z}{\sqrt{10^4z^2+100}}\right),
\quad
v(z)=q_5\left(\frac1{\sqrt{10^4z^2+100}}\right),
\quad
\Psi(z)=\frac{v(z)}{u(z)}.
\]
With $u_t=u(z_t)$ and $v_t=v(z_t)$, the NS direction is
\[
D_t=\operatorname{diag}
\bigl(\operatorname{sgn}(x_t)u_t,
      \operatorname{sgn}(y_t)v_tI_{100}\bigr),
\]
 where $\operatorname{sgn}(\cdot)$ denotes the sign function.
Since $x_t=\operatorname{sgn}(x_t)|x_t|$ and
$y_t=\operatorname{sgn}(y_t)|y_t|$,  we have
\[
f(W_t-\eta D_t)
=50\big(
(|x_t|-\eta u_t)^2+(|y_t|-\eta v_t)^2
\big)
\]
Differentiating with respect to $\eta$ gives
\[
\frac{\mathrm d}{\mathrm d\eta}f(W_t-\eta D_t)
=100\big(
\eta(u_t^2+v_t^2)-|x_t|u_t-|y_t|v_t
\big),~
\eta_t^\star
=\frac{|x_t|u_t+|y_t|v_t}{u_t^2+v_t^2}.
\]

Substitution into the update gives
\begin{align*}
x_{t+1}
&=\operatorname{sgn}(x_t)
\frac{v_t(|x_t|v_t-|y_t|u_t)}{u_t^2+v_t^2},\\
y_{t+1}
&=\operatorname{sgn}(y_t)
\frac{u_t(|y_t|u_t-|x_t|v_t)}{u_t^2+v_t^2}.
\end{align*}

If $|x_t|v_t-|y_t|u_t=0$, both new blocks vanish and the
algorithm reaches the minimizer.
Otherwise,  
\begin{align}
    z_{t+1}
&=\frac{|x_{t+1}|}{|y_{t+1}|}
=\frac{v_t}{u_t}
=\Psi(z_t),  \nonumber\\
\frac{f(W_{t+1})}{f(W_t)}
&=\frac{x_{t+1}^2+y_{t+1}^2}{x_t^2+y_t^2}
=\frac{(u_t^2+v_t^2)(|x_t|v_t-|y_t|u_t)^2}
{(x_t^2+y_t^2)(u_t^2+v_t^2)^2} \nonumber\\
&=\frac{(|x_t|v_t-|y_t|u_t)^2}
{(x_t^2+y_t^2)(u_t^2+v_t^2)}
=\frac{|y_t|^2u_t^2
(z_t\frac{v_t}{u_t}-1)^2}
{|y_t|^2(1+z_t^2)\,
u_t^2(1+(\frac{v_t}{u_t})^2)} \nonumber\\
&=\frac{(1-z_t\Psi(z_t))^2}
{(1+z_t^2)(1+\Psi(z_t)^2)}. \label{eq:two-level-recursion}
\end{align}

We next find an interval $J$ containing $z_0=1$ such that
$\Psi(J)\subset J$.  
We also bound $z\Psi(z)$ on $J$ to derive the convergence bounds
for the objective value.

 \medskip
\noindent\emph{1) Practical NS.}
We show that $z_t\in[0.9920,1]$ for all $t$, which keeps the
numerator in \eqref{eq:two-level-recursion} small.

Suppose $z\in[0.9920,1]$. The normalized singular values satisfy
\[
\frac{100z}{\sqrt{10^4z^2+100}}
\in[0.9949\ldots,0.9950\ldots],
\qquad
\frac1{\sqrt{10^4z^2+100}}
\in[0.0099\ldots,0.0100\ldots].\]
On these two intervals, $q_{\mathrm P,5}$ is decreasing and
increasing, respectively, as verified by direct interval propagation
through the five NS steps, following the same argument as in
the proof of Lemma~\ref{lem:ns}.
A direct calculation gives
\[
0.7020\ldots\le u(z)\le0.7021\ldots,
\qquad
0.6967\ldots\le v(z)\le0.7002\ldots.
\]
Consequently,
\[
0.9920
<
\frac{0.6967\ldots}{0.7021\ldots}
\le\Psi(z)
\le\frac{0.7002\ldots}{0.7020\ldots}
<1.
\]
Since $z_0=1$ and $z_{t+1}=\Psi(z_t)$, induction
gives $z_t\in[0.9920,1]$ for every $t$.
In particular,
\[
0.9920^2\le z_t\Psi(z_t)<1.
\]
Applying \eqref{eq:two-level-recursion}, we obtain
\[
\frac{f(W_{t+1}^{\mathrm P})}{f(W_t^{\mathrm P})}
\le
\frac{(1-0.9920^2)^2}{(1+0.9920^2)^2}
\le 6.4514\times10^{-5}.
\]
Multiplying the above inequality yields the desired result.

\medskip
\noindent\emph{2) Standard NS.}
We show that $z_t\in J_{\mathrm S}=[1/5,1]$ for all $t$.
Suppose $z\in[1/5,1]$.  The normalized singular values satisfy
$1/\sqrt{10^4z^2+100}\ge s_0=1/\sqrt{10100}$.
Recall that $a_{\mathrm S}=15/8$. By \eqref{both:standard},  
\[
s\le q_{\mathrm S,5}^{(2)}(s)
\le\min\{a_{\mathrm S}^5s,1\},\qquad s\in[0,1].
\]
For the first four NS step ($j=0,\ldots,4$), we have
\[
q_{\mathrm S,j}^{(2)}(s_0)
\le a_{\mathrm S}^js_0
<\frac{a_{\mathrm S}^4}{100}<\frac18.
\]
Thus, at each of the five NS steps, we can apply
\[
q_{\mathrm S}^{(2)}(s)
=\frac{15}{8}s-\frac54s^3+\frac38s^5
\ge\left(\frac{15}{8}-\frac5{256}\right)s
=\frac{475}{256}s,
\qquad s\in[0,1/8],
\]
which gives
\[
q_{\mathrm S,5}^{(2)}(s_0)
\ge\big(\frac{475}{256}\big)^5s_0
>\frac1{101}\big(\frac{475}{256}\big)^5
>\frac15.
\]
For $z\in J_{\mathrm S}$, the normalized singular values satisfy
$s_0\le 1/\sqrt{10^4z^2+100}\le 100z/\sqrt{10^4z^2+100}\le1$.
Monotonicity of $q_{\mathrm S,5}^{(2)}$ therefore yields
\[
\frac15<v(z)\le u(z)\le1,
\qquad
\frac15<\Psi(z)=\frac{v(z)}{u(z)}\le1,
\]
which proves the claimed inclusion. Moreover,
\[
z\Psi(z)
\le z\frac{a_{\mathrm S}^5/\sqrt{10^4z^2+100}}{100z/\sqrt{10^4z^2+100}}
=\frac{a_{\mathrm S}^5}{100}<\frac14.
\]
Starting from $z_0=1$, \eqref{eq:two-level-recursion} thus gives,
\[
\frac{f(W_{t+1}^{\mathrm S})}{f(W_t^{\mathrm S})}
=\frac{(1-z_t\Psi(z_t))^2}
{(1+z_t^2)(1+\Psi(z_t)^2)}
\ge\frac{(1-1/4)^2}{4}
=\frac9{64}=0.140625.
\]
Multiplying the above inequality yields the desired result.
\end{proof}



\clearpage

\section{Experimental details on ResNet-32}\label{app:convergence}


This appendix details the ResNet-32 experiments used to examine Muon's convergence rate and its advantage over AdamW in dimension dependence, 
including the experimental setup, gradient norm convergence rates, accuracy and loss results, and hyperparameter sweeps.

\subsection{Experimental setting}

\paragraph{Models and data.}



We use ResNet-32 models of varying dimensions measured by the number of model parameters.
By changing the base channel width
$w\in\{16,\allowbreak 32,\allowbreak 64,\allowbreak
192,\allowbreak 256,\allowbreak 320\}$,
we obtain models with
$\{0.47,\allowbreak 1.86,\allowbreak 7.41,\allowbreak 66.46,\allowbreak 118.11,\allowbreak 184.50\}$ million parameters,
respectively.
For each choice of $w$, the model consists of an initial $3\times3$ convolution followed by BatchNorm
and ReLU, and three residual stages with channel widths $w$, $2w$, and $4w$,
respectively.
Each stage contains five residual blocks, each comprising two
$3\times3$ convolutions with BatchNorm, while the first block of the second and
third stages downsamples the feature map.
The resulting features are then globally
average pooled and passed through a linear classifier to produce logits for
100 classes. 
All models are trained from random initialization.

We train the models on CIFAR-100 \citep{krizhevsky2009learning},
using the standard training set containing 50,000 images
and test set containing 10,000 images.
All images are normalized per channel, with no data augmentation.

\paragraph{Compared optimizers and hyperparameters.}

We compare three optimizers: Muon with practical NS and WD,
Muon with standard NS ($k=2$) and WD, and AdamW.

Muon is applied to convolutional and fully connected weight matrices,
with convolutional kernels reshaped into two-dimensional matrices,
while an auxiliary AdamW optimizer updates biases and normalization parameters
with a fixed stepsize of $10^{-3}$ and zero WD.
Both Muon variants use five NS steps,
momentum coefficient $\beta=0.95$, and normalization floor
$\varepsilon=10^{-7}$.
For the AdamW baseline, AdamW updates all parameters,
with biases and normalization parameters excluded from WD.
It uses first- and second-moment coefficients
$(\beta_1,\beta_2)=(0.9,0.999)$ and numerical stability constant
$\hat{\varepsilon}=10^{-8}$.
The selected stepsizes and WD coefficients are reported in
Table~\ref{tab:resnet32-selected-hyperparameters},
and the stepsize schedule consists of a $5\%$ linear warmup followed by
cosine decay to zero.
Specifically, for each dimension and optimizer, we jointly sweep the stepsize and WD coefficient,
training each setting for 2,000 iterations with a batch size of 128 and a fixed
random seed, and select the setting with the lowest final validation loss.
The resulting sweep landscapes are shown in
Appendix~\ref{subapp:res:sweep} (Figure~\ref{fig:resnet_sweep}).
The final experiments use three random seeds.
The forward and backward passes do not use BF16 autocast, while model
parameters and optimizer states are kept in FP32,
with CUDA TF32 operations enabled.
Neither gradient clipping nor shape dependent Muon scaling is used.




\begin{table}[H]
\centering
\caption{Selected stepsize and WD pairs for ResNet-32.}
\label{tab:resnet32-selected-hyperparameters}
\small
\begin{tabular}{lccc}
\hline
Dimension (M) &  Muon (Practical NS + WD) & Muon (Standard NS + WD)& AdamW \\
\hline
0.47   & $(0.1,0.0625)$  & $(0.1,0.0625)$   & $(0.01,0.0625)$ \\
1.86   & $(0.05,0.1)$    & $(0.05,0.1)$     & $(0.01,0.15)$ \\
7.41   & $(0.02,0.3)$    & $(0.05,0.1)$     & $(0.001,0.04)$ \\
66.46  & $(0.05,0.15)$   & $(0.0075,0.35)$  & $(0.001,0.0625)$ \\
118.11 & $(0.01,0.6)$    & $(0.1,0.0625)$   & $(0.0003,0.1)$ \\
184.50 & $(0.02,0.4)$    & $(0.01,0.35)$    & $(0.0003,0.2)$ \\
\hline
\end{tabular}
\end{table}


\paragraph{Evaluation metric.}
We measure stationarity at every iteration 
using the Frobenius norm of the gradient computed over the entire training set, $\|\nabla f(W)\|_F$.
To compare the dimension dependence of Muon and AdamW, at iteration
2,000, we report
$\sqrt{d}\|\nabla f(W_{\mathrm{Muon}})\|_F/
\|\nabla f(W_{\mathrm{AdamW}})\|_F$,
where  $W_{\mathrm{Muon}}$ and $W_{\mathrm{AdamW}}$ denote the parameters
obtained by Muon and AdamW, respectively.
If this ratio remains approximately constant as $d$ increases,
the results support Muon's $\sqrt{d}$ improvement in dimension dependence
over AdamW established by our theoretical comparison.
Training and test losses are recorded every 100 iterations, with final top-1 accuracy evaluated on the full test set.

\clearpage
\subsection{Gradient norm convergence rates}

\begin{figure}[h]
    \centering 
    \includegraphics[width=1.0\linewidth]{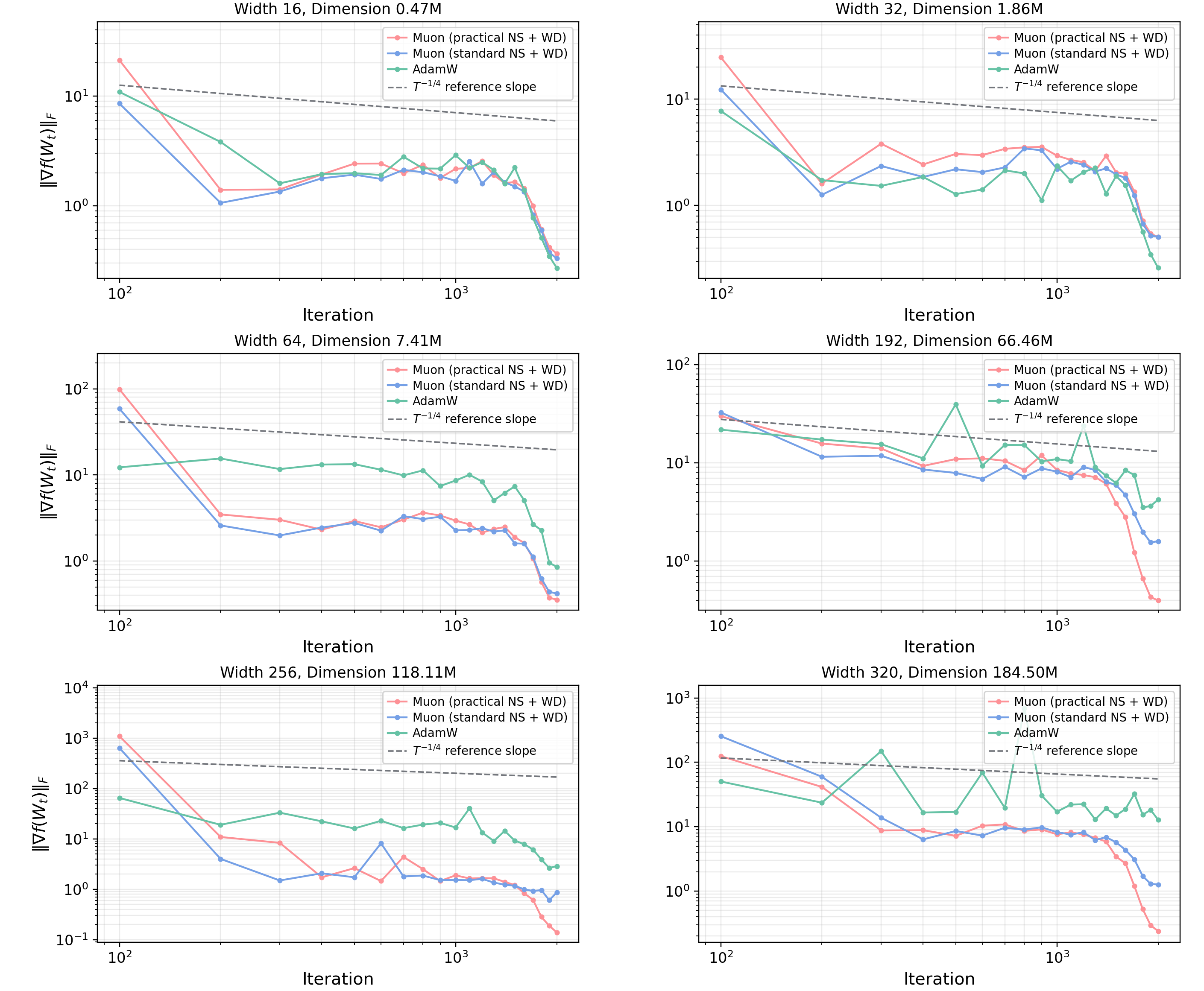}
    \caption{Frobenius gradient norm convergence rates for ResNet-32 models of varying dimensions.
    The gradient norms exhibit an overall $T^{-1/4}$ trend, with Muon generally converging faster than AdamW.}
\end{figure}

\subsection{Accuracy and loss results}

\begin{table}[H]
\centering
\caption{CIFAR-100 test top-1 accuracy (
over three seeds) for ResNet-32 models of varying dimensions after 2,000 iterations.
Muon consistently achieves higher accuracy than AdamW.}
\label{tab:resnet32-top1-accuracy}
\scriptsize
\setlength{\tabcolsep}{3pt}
\resizebox{\linewidth}{!}{%
\begin{tabular}{lcccccc}
\hline
Algorithm \textbackslash{} Dimension (M)   &   0.47 & 1.86 &   7.41 &    66.46 &   118.11 &  184.50 \\
\hline
Muon (practical + WD) & $59.64\pm0.21$ & $66.35\pm0.38$ & $70.22\pm0.25$ & $73.66\pm0.18$ & $72.74\pm0.62$ & $74.57\pm0.04$ \\
 Muon (standard + WD)  & $58.71\pm0.36$ & $65.59\pm0.31$ & $68.97\pm0.38$ & $72.28\pm0.20$ & $70.88\pm0.40$ & $72.92\pm0.15$ \\
AdamW          & $45.95\pm1.15$ & $52.29\pm1.88$ & $57.03\pm0.55$ & $59.79\pm1.61$ & $62.69\pm0.37$ & $63.17\pm0.62$ \\
\hline
\end{tabular}%
}
\end{table}

\begin{figure}[h]
    \centering 
    \includegraphics[width=.9\linewidth]{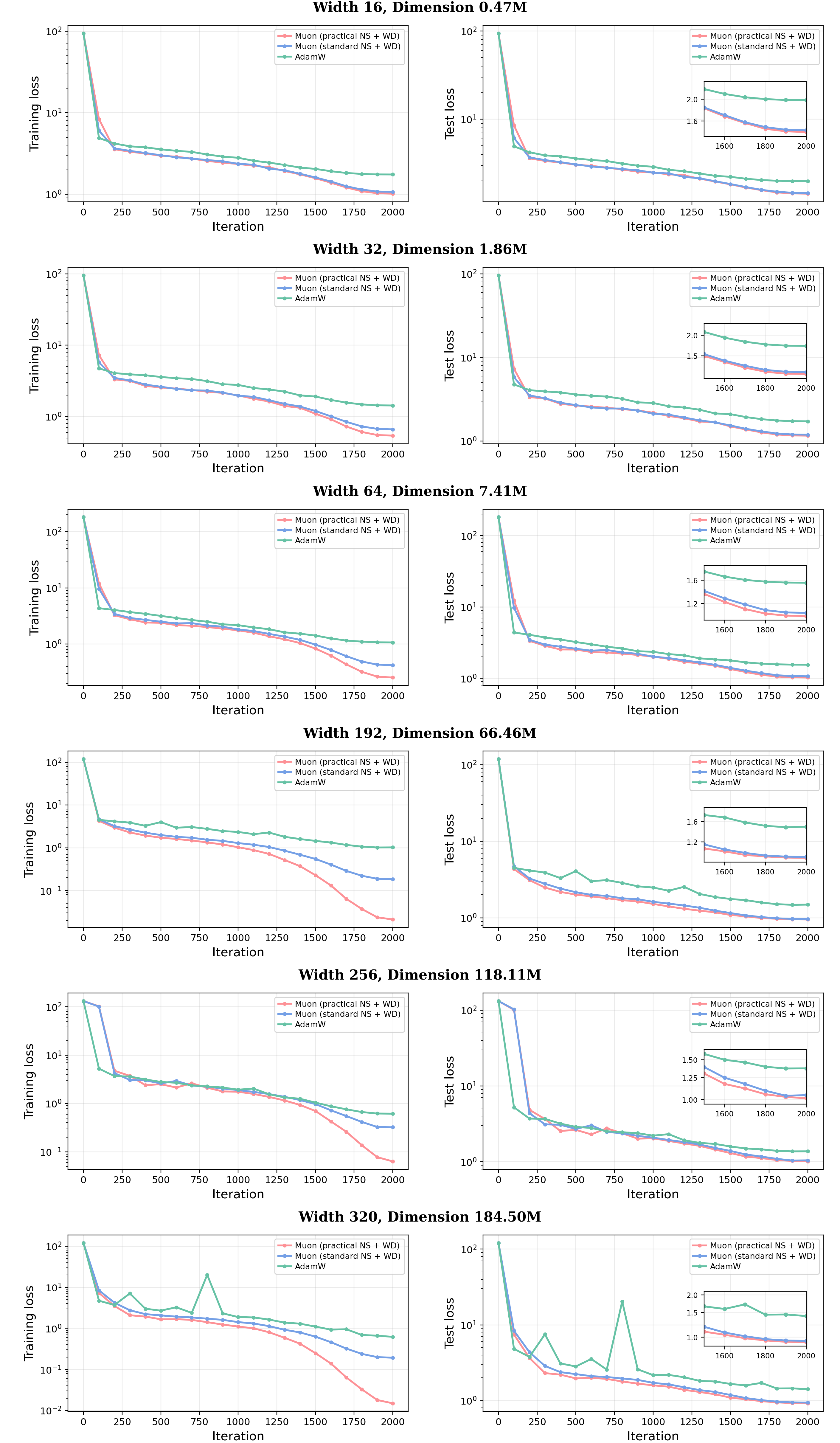}
    \caption{Loss curves for ResNet-32 models of varying dimensions. Muon consistently achieves lower training and test losses than AdamW.}
    \label{fig:resnet_loss}
\end{figure}

\clearpage

\subsection{Hyperparameter sweep} \label{subapp:res:sweep}

\begin{figure}[h]
    \centering 
    \includegraphics[width=0.79\linewidth]{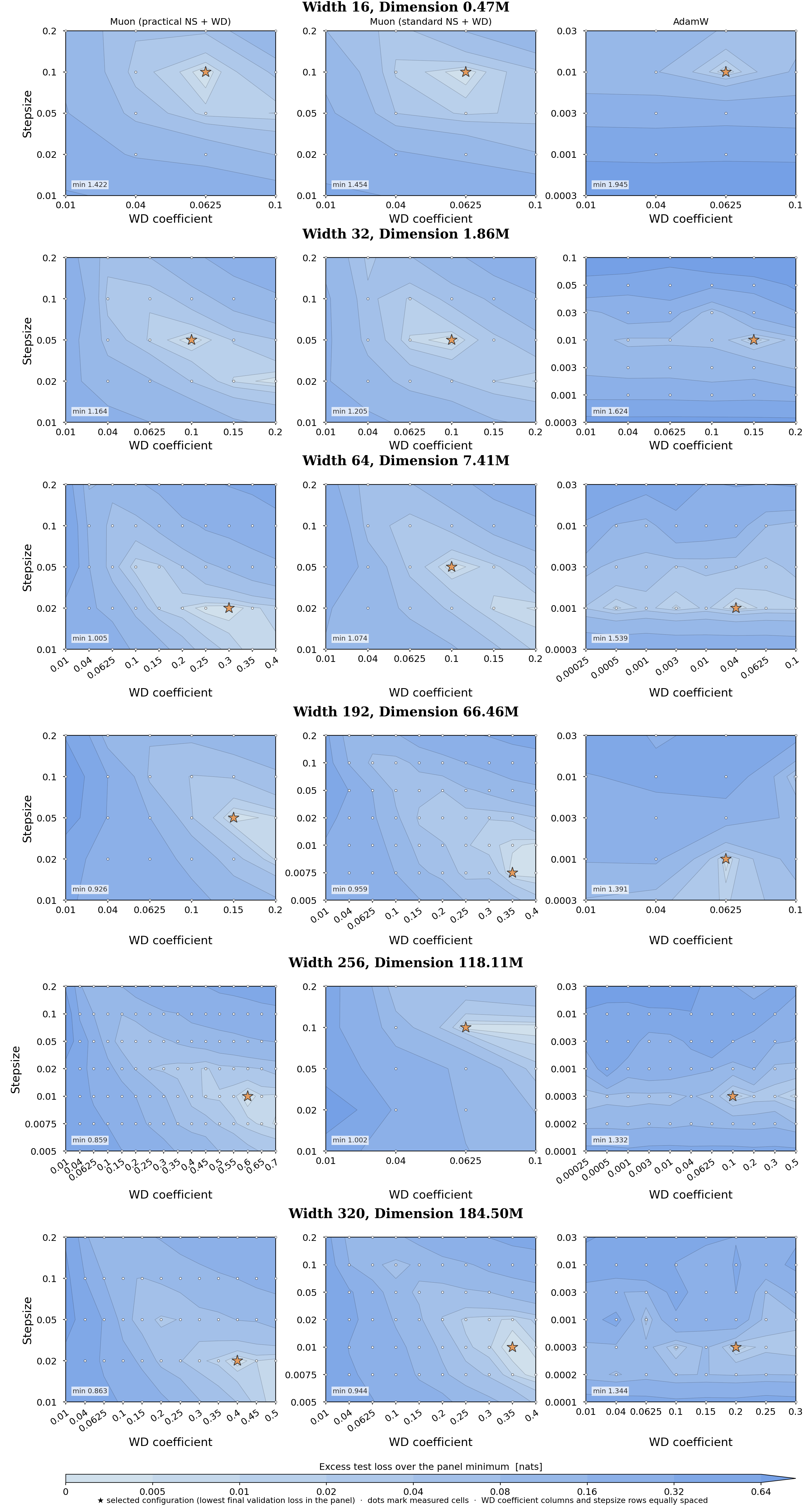}
    \caption{Joint sweep over stepsize and WD for ResNet-32 models of varying dimensions.}
    \label{fig:resnet_sweep}
\end{figure}

\clearpage

\section{Experimental details on language models} \label{app:error}

This appendix details the experiments on the 124M GPT-2-style model
and presents additional results on the larger 355M GPT-2-style and 0.6B Qwen3-style models \citep{yang2025qwen3},
further showing that large factors arising from worst-case NS orthogonalization errors need not slow Muon's convergence.
Specifically, we provide the experimental setup, loss curves, singular value distributions, maximum orthogonalization errors and the resulting factors, and hyperparameter sweeps.

\subsection{Experimental setting}
\paragraph{Models and data.}
We use three models: 124M- and 355M-parameter GPT-2-style models, and a 0.6B-parameter Qwen3-style model.
The 124M model has 12 layers, hidden size 768, and 12 attention heads.  The 355M model has 24 layers, hidden
size 1,024, and 16 attention heads.  The 0.6B model has 28 layers, hidden size
1,024, 16 query heads, 8 key-value heads, and an intermediate size of 3,072.
All models are trained from random initialization.

We train the 124M, 355M, and 0.6B models for 500M, 200M, and 300M tokens, respectively, 
and evaluate them on a separate validation set containing 5M tokens.
Training data come from the FineWeb-Edu sample-10BT subset \citep{penedo2024fineweb}.
The 124M and 355M GPT-2-style models use the GPT-2 tokenizer, while the 0.6B Qwen3-style model uses the Qwen3 tokenizer.
All models use a sequence length of 512 tokens,
 and a batch size of 131,072 tokens per optimizer update.

\paragraph{Task and metrics.}

We consider autoregressive language modeling with a next-token prediction
objective and report training and validation cross-entropy losses.
To illustrate that the worst-case orthogonalization error does not
necessarily imply slow convergence, we track the singular-value
distributions and orthogonalization errors $\epsilon_{t,5}$
of selected momentum matrices, along with the factors
$(1-\epsilon_{t,5})^{-1}$, motivated by the factor
$(1-\epsilon_5)^{-1}$ in the convergence analysis of standard NS
without WD by \citet{kim2026muon}.

Training loss is recorded at every iteration, and validation loss is
evaluated every 200 iterations for the 124M model and every 100 iterations
for the 355M and 0.6B models. Each evaluation is performed using
20 batches for the 124M and 355M GPT-2-style models and 10 batches for the 0.6B Qwen3-style model.

The singular values of momentum matrices for representative
attention and MLP weights in the first, middle, and last transformer
blocks are computed using FP64 SVDs without affecting training
and recorded at every iteration.
Let $s_{i,t}^{l}$ denote the $i$-th normalized singular value
of momentum matrix $l$ at iteration $t$.
The orthogonalization error $\epsilon_{t,5}$ is computed as follows:
\[
\epsilon_{t,5}^{(\hat{a})}
=
\max_{l}\max_i
\left|q_{\hat{a},5}(s_{i,t}^{l})-1\right|,
\qquad \hat{a}\in\{\mathrm P, \mathrm S \}.
\]

\paragraph{Compared optimizers and hyperparameters.}
We use the three optimizers as in
Appendix~\ref{app:convergence}.
Muon is applied to two-dimensional hidden-layer matrices, while
embeddings, output heads, biases, and normalization parameters
are updated by an auxiliary AdamW optimizer.
The auxiliary AdamW uses a fixed stepsize of $6\times10^{-4}$ for the 124M model
and $3\times10^{-4}$ for the 355M and 0.6B models.
It applies the selected WD of the corresponding Muon variant to embeddings and output-head matrices,
while biases and normalization parameters use zero WD.
For the AdamW baseline, all parameters are updated by AdamW,
with the selected WD applied to matrix parameters
and zero WD applied to biases and normalization parameters.

The optimizer settings follow Appendix~\ref{app:convergence}, except where otherwise specified.
AdamW uses $\beta_2=0.95$.
The selected stepsizes and WD coefficients are reported in Table~\ref{tab:selected-hparams},
and the stepsize schedule consists of a $2\%$ linear warmup followed by cosine decay to $10\%$ of the selected value.
Specifically, for each model and optimizer, we jointly sweep the stepsize and WD coefficient,
training each setting with a fixed random seed for $30\%$ of the corresponding
training schedule, i.e., 1,145, 458, and 687 iterations for the 124M, 355M,
and 0.6B models, respectively, and select the setting with the lowest final
validation loss. The resulting sweep landscapes are shown in
Appendix~\ref{subapp:sweep} (Figure~\ref{fig:wd-momentum-sweep}).
The final experiments use three seeds for the 124M and 355M models and one seed for the 0.6B model.
The forward and backward passes use BF16 autocast.

\begin{table}[htbp]
\centering
\caption{Selected stepsize and WD pairs for language model training.}
\label{tab:selected-hparams}
\small
\begin{tabular}{lccc}
\hline
Model & Muon (Practical NS + WD) & Muon (Standard NS + WD) & AdamW \\
\hline
124M GPT-2-style  & $(0.02,0.000625)$ & $(0.02,0.0025)$ & $(0.001,0.01)$ \\
355M GPT-2-style  & $(0.02,0.02)$ & $(0.02,0.0025)$ & $(0.0003,0.00125)$ \\
0.6B Qwen3-style & $(0.04,0.0025)$ & $(0.02,0.0025)$ & $(0.0015,0.000625)$ \\
\hline
\end{tabular}
\end{table}



\newpage
\subsection{Loss curves}

\begin{figure}[h]
    \centering 
    \includegraphics[width=1.0\linewidth]{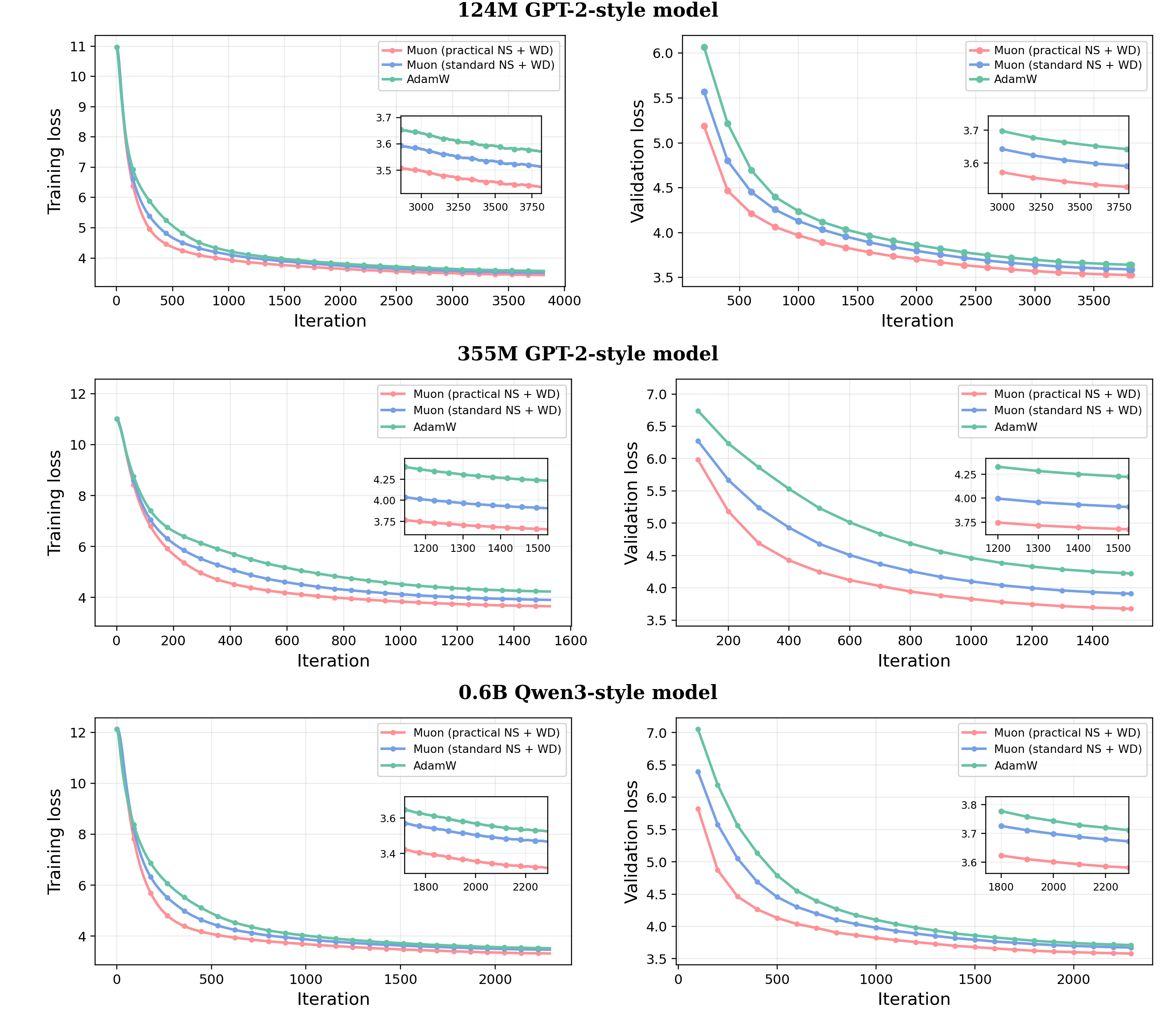}
     \caption{Loss curves across the three language models.
     Muon consistently achieves lower training and validation losses than AdamW.}
    \label{fig:larger_loss}
\end{figure}

\clearpage

\subsection{Orthogonalization error}

\begin{figure}[h]
    \centering 
    \includegraphics[width=1.0\linewidth]{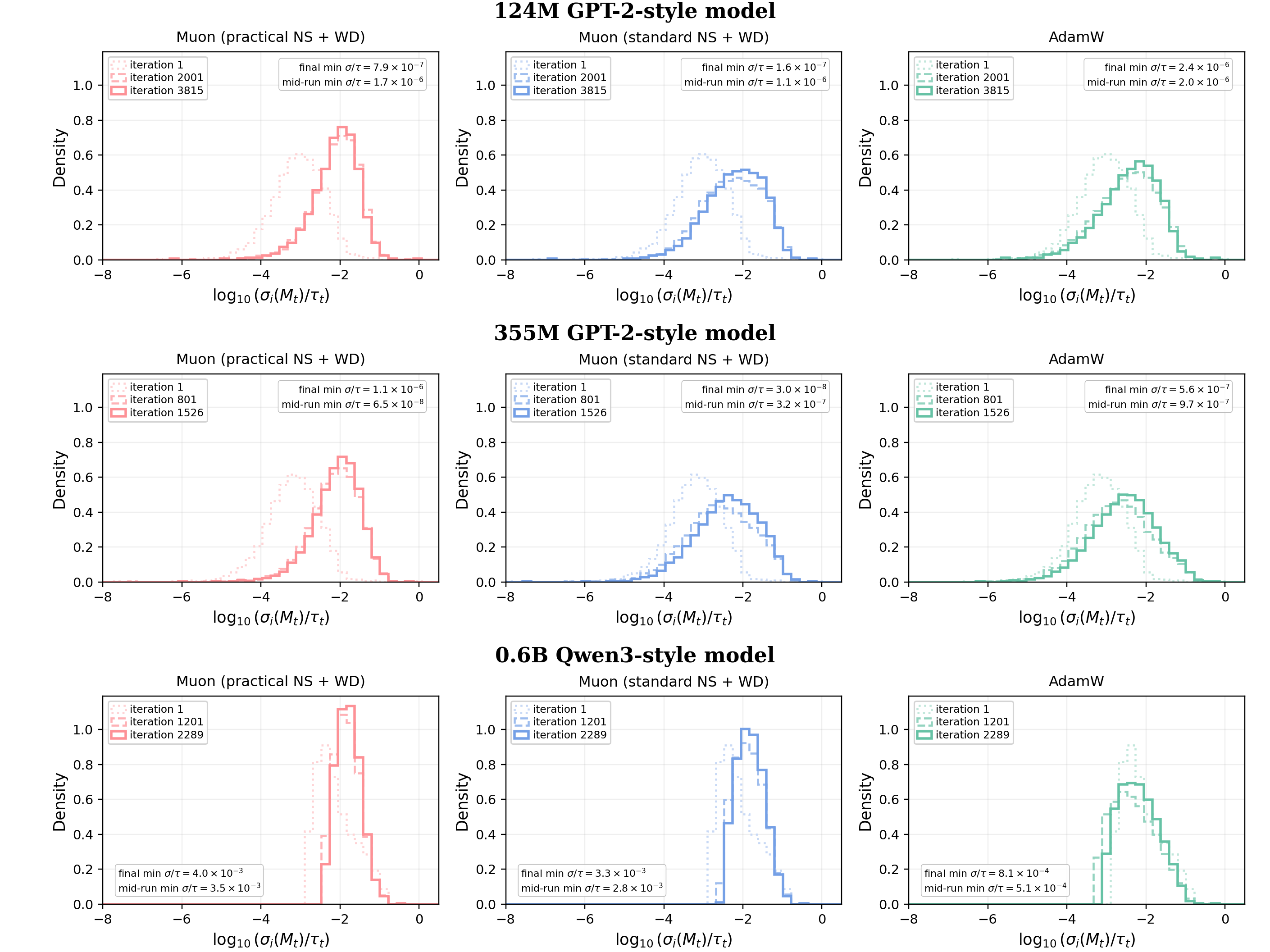}
    \caption{
    Momentum singular-value distributions for the query projection
in the middle transformer block of the three language models.
The distributions span a broad range and include very small normalized singular values.}
    \label{fig:sigular}
\end{figure}

\begin{figure}[h]
    \centering 
    \includegraphics[width=1.0\linewidth]{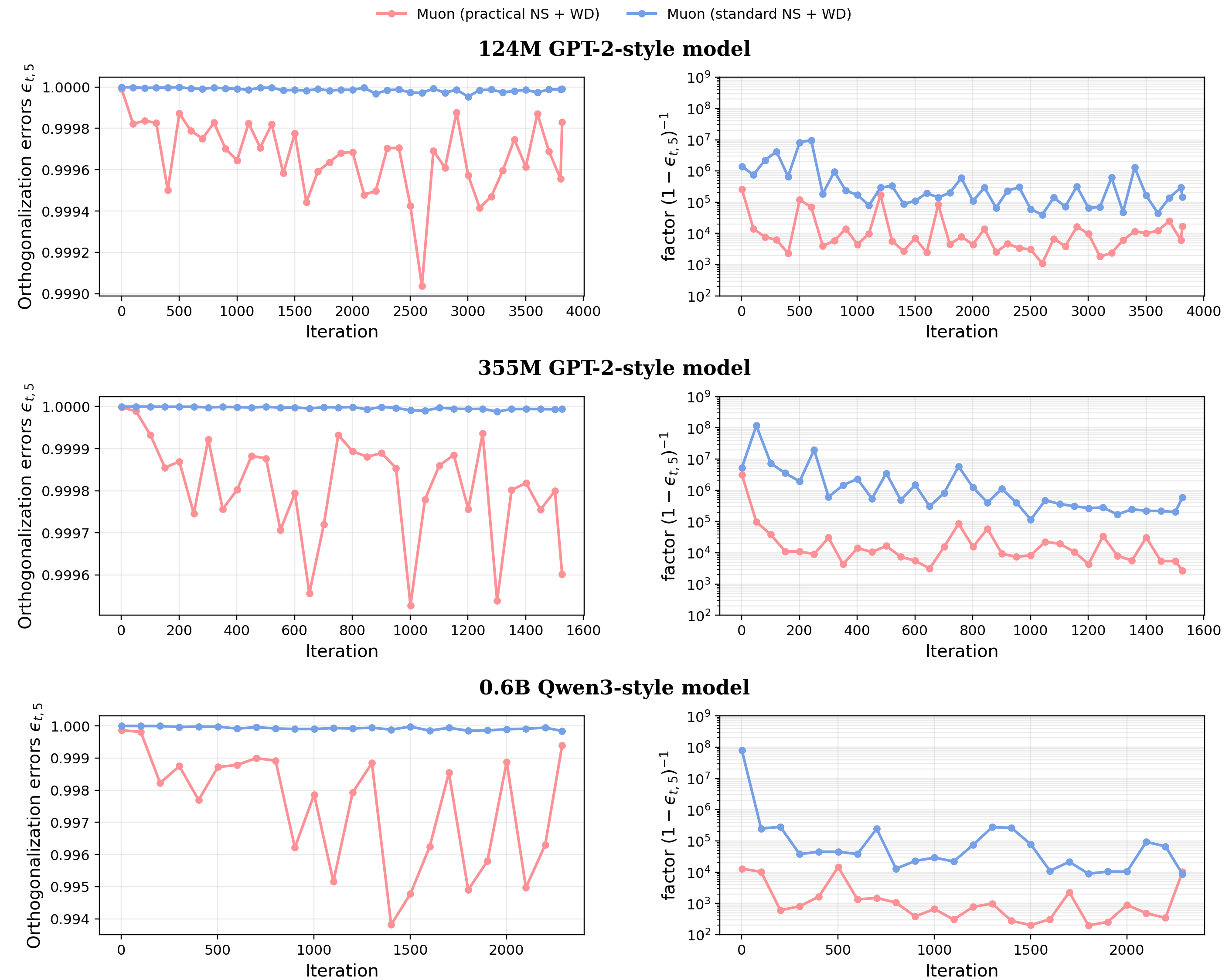}
    \caption{
    Orthogonalization errors $\epsilon_{t,5}$ and the corresponding
    factors $(1-\epsilon_{t,5})^{-1}$ for practical and standard NS across
    the three language  models. Practical NS consistently exhibits smaller errors and
    factors, while the large factor values need not imply slow convergence.
    }
    \label{fig:error:factor}
\end{figure}

\clearpage
\subsection{Hyperparameter sweep} \label{subapp:sweep}

\begin{figure}[h]
    \centering
    \includegraphics[width=\linewidth]{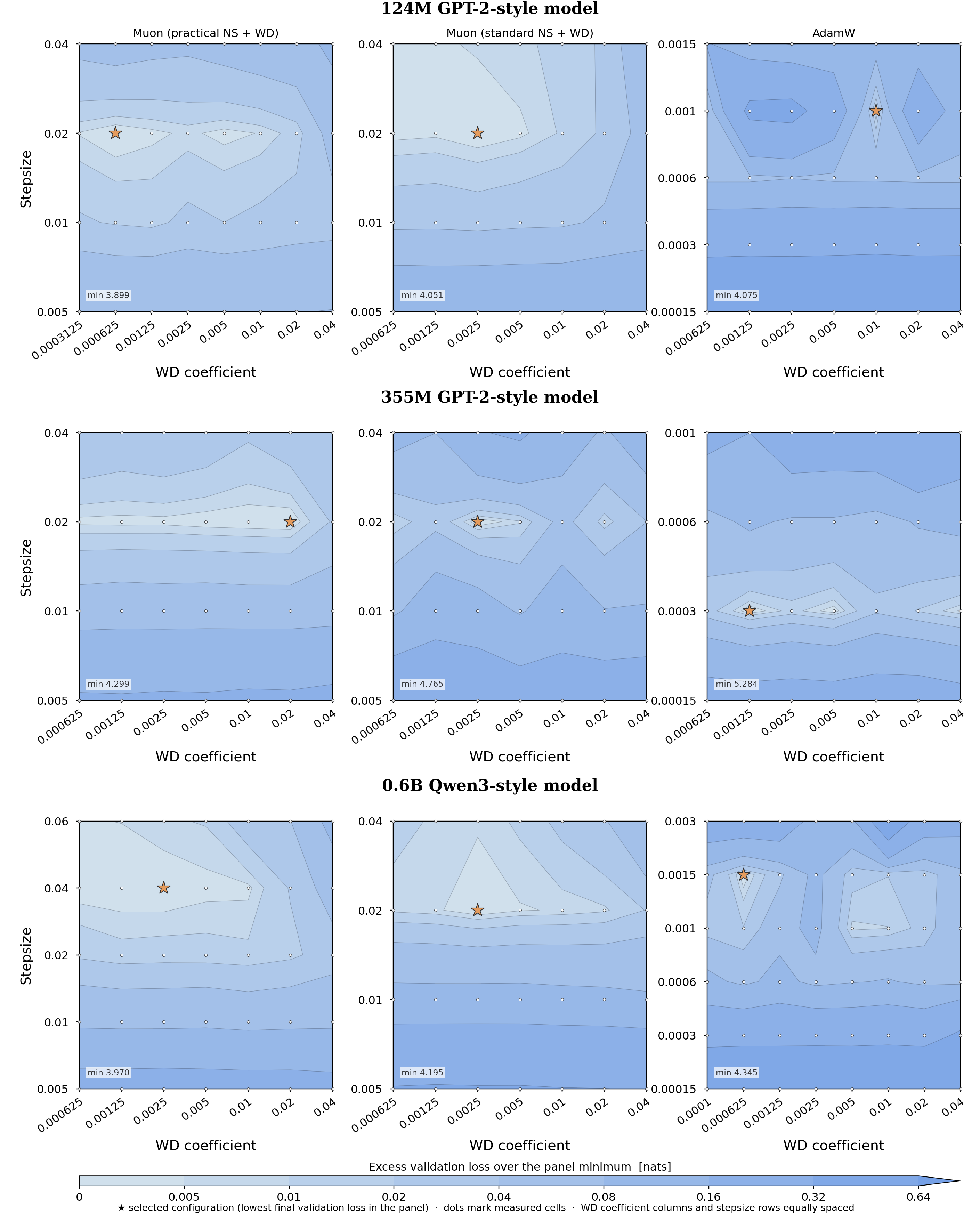}
    \caption{Joint sweep over stepsize and WD for language model training.}    \label{fig:wd-momentum-sweep}
\end{figure}

\newpage

\section{Interpreting standard NS with WD}

As discussed in Section~\ref{sec:interpret}, Muon with standard NS admits a more direct optimization interpretation through its fixed-point condition, 
since $q_{\mathrm S,5}$ is strictly increasing on $[0,1]$ and hence invertible.

\paragraph{Dynamic objective.}
Recall the scalar fixed-point condition 
$
q_{\mathrm S,5}(G_t/\tau_t)+\lambda W_t=0
$
from
Section~\ref{sec:interpret}, where
$q_{\mathrm S,5}$ denotes $q_{\mathrm S,5}^{(k)}$
for arbitrary $k\in\mathbb N$.
Applying the odd inverse $q_{\mathrm S,5}^{-1}$ gives
\[
G_t+\tau_t q_{\mathrm S,5}^{-1}(\lambda W_t)=0.
\]
We obtain a regularizer by integrating the second term on the
left-hand side with respect to $W$.
Extending this construction to matrices yields
\begin{equation}
\widetilde\Phi_t(W)
=f(W)+\widetilde{\mathcal R}_t(W),
\qquad
\widetilde{\mathcal R}_t(W)
=\frac{\tau_t}{\lambda}
\sum_{i=1}^{r}
\int_0^{\lambda\sigma_i(W)}
\tilde q_{\mathrm S,5}^{-1}(s)\,\mathrm ds.
\label{eq:standard-dynamic-objective}
\end{equation}
Here, $\tau_t$ is held fixed when differentiating in $W$,
and $\sigma_i(W)$ denotes the $i$-th singular value of $W$.
The extended inverse is defined by
\[
\tilde q_{\mathrm S,5}^{-1}(s)
=
\begin{cases}
q_{\mathrm S,5}^{-1}(s), & 0\le s\le1,\\
s, & s>1,
\end{cases}
\]
where $q_{\mathrm S,5}^{-1}:[0,1]\to[0,1]$ denotes the inverse
of $q_{\mathrm S,5}$ restricted to $[0,1]$.
For negative arguments, we use the odd extension.
Alternatively, one may set $\tilde q_{\mathrm S,5}^{-1}(s)=+\infty$ for $s>1$, but we omit discussion of this case for brevity.
The following proposition establishes the equivalence between
the Muon update using standard NS and optimization of the dynamic objective $\widetilde\Phi_t(W)$.
\begin{proposition}
\label{prop:standard-dynamic-regularization}

\textbf{(i) Properties of the regularizer.}
\begin{enumerate}[label=\arabic*),leftmargin=*,nosep]
\item
The regularizer $\widetilde{\mathcal R}_t$ in
\eqref{eq:standard-dynamic-objective} is continuously
differentiable and $\lambda\tau_t/C_1$-strongly convex.

\item
For an SVD
$W=U_W\operatorname{diag}(\sigma_1(W),\ldots,\sigma_r(W))V_W^\top$,
\[
\nabla\widetilde{\mathcal R}_t(W)
=\tau_t U_W\operatorname{diag}\Bigl(
\tilde q_{\mathrm S,5}^{-1}(\lambda\sigma_1(W)),\ldots,
\tilde q_{\mathrm S,5}^{-1}(\lambda\sigma_r(W))
\Bigr)V_W^\top.
\]

\item
The following bounds hold:
\[
\frac{\lambda\tau_t}{2C_1}\|W\|_F^2
\le\widetilde{\mathcal R}_t(W)
\le\frac{\lambda\tau_t}{2}\|W\|_F^2,
\qquad
\|\nabla\widetilde{\mathcal R}_t(W)\|_F
\le\lambda\tau_t\|W\|_F.
\]
\end{enumerate}

\textbf{(ii) Fixed-point equivalence and descent alignment.}
Ignoring the momentum by setting $\beta=0$,
Muon update in Algorithm~\ref{alg:muon} 
using standard NS has the following properties:
\begin{enumerate}[label=\arabic*),leftmargin=*,nosep]
\item
The update direction and the stochastic gradient vanish
simultaneously, i.e.,
\[
H_t=0
\quad\Longleftrightarrow\quad
\widehat{\nabla}\widetilde\Phi_t(W_t)=0, \qquad \widehat{\nabla}\widetilde\Phi_t(W_t)
=G_t+\nabla\widetilde{\mathcal R}_t(W_t).
\]
\item
The update direction $H_t$ is positively aligned with
the stochastic gradient $\widehat{\nabla}\widetilde\Phi_t(W_t)$, i.e.,
\[
\langle
\widehat{\nabla}\widetilde\Phi_t(W_t),H_t
\rangle_F
\ge \frac{\tau_t}{C_1}\|H_t\|_F^2.
\]

\end{enumerate}

\noindent\textbf{(iii) Exact generalized Frank--Wolfe equivalence.}
With $\beta=0$, and $\gamma=\eta\lambda\in(0,1]$,  Muon
update in Algorithm~\ref{alg:muon} using standard NS is exactly
a stochastic generalized Frank--Wolfe step on the dynamic
objective \eqref{eq:standard-dynamic-objective},
with stepsize $\gamma$.

Specifically, Muon selects its update direction by seeking
the greatest decrease in the regularized linear model,
whose unique minimizer is
\begin{equation}
\operatorname*{arg\,min}_{W\in\mathbb R^{m\times n}}
\left\{
\langle G_t,W\rangle_F+\widetilde{\mathcal R}_t(W)
\right\}
=-\frac{D_t}{\lambda}.
\label{eq:standard-gfw-oracle}
\end{equation}
The Muon update therefore takes the equivalent form
\begin{equation}
W_{t+1}
=(1-\gamma)W_t
+\gamma\big(-\frac{D_t}{\lambda}\big),
\label{eq:standard-gfw-update}
\end{equation}
which moves a fraction $\gamma$ of the way from $W_t$
to the unique minimizer $-D_t/\lambda$ of the regularized
linear subproblem.

\textbf{(iv) Recovery of the original objective.}
Consider the deterministic setting where $G_t=\nabla f(W_t)$.
If $\|\nabla f(W_t)\|_F\to0$, then $\tau_t\to\varepsilon$.
Together with the
boundedness of $\{W_t\}$, which WD naturally promotes, the bounds in (i) imply that
both $\widetilde{\mathcal R}_t(W_t)$ and
$\|\nabla\widetilde{\mathcal R}_t(W_t)\|_F$ are asymptotically
$\mathcal O(\varepsilon)$.
Thus, $\widetilde\Phi_t$ approximates the original objective
near stationarity.

\end{proposition}

\begin{proof}
We first establish bounds on $\tilde q_{\mathrm S,5}^{-1}$.
The standard NS map satisfies
\[
q_{\mathrm S}'(s)=q_{\mathrm S}'(0)(1-s^2)^k.
\]
Indeed, let $\omega_\ell=(2\ell)!/[4^\ell(\ell!)^2]$.
Differentiating
$q_{\mathrm S}(s)=s\sum_{\ell=0}^{k}\omega_\ell(1-s^2)^\ell$
and using $(2\ell+1)\omega_\ell=2(\ell+1)\omega_{\ell+1}$ gives
\begin{align*}
q_{\mathrm S}'(s)
&=\sum\nolimits_{\ell=0}^{k}\omega_\ell(1-s^2)^\ell   + 
\big(2(1-s^2)-2\big)\sum\nolimits_{\ell=0}^{k}\omega_\ell\ell(1-s^2)^{\ell-1}          \\
&=\sum\nolimits_{\ell=0}^{k}(2\ell+1)\omega_\ell(1-s^2)^\ell
 -\sum\nolimits_{\ell=0}^{k-1}2(\ell+1)\omega_{\ell+1}(1-s^2)^\ell\\
&=(2k+1)\omega_k(1-s^2)^k.
\end{align*}
Evaluating at $s=0$ yields
$q_{\mathrm S}'(0)=(2k+1)\omega_k$, proving the identity.
The chain rule then gives
\[
0<q_{\mathrm S,5}'(s)
=\prod\nolimits_{j=0}^{4}q_{\mathrm S}'(q_{\mathrm S,j}(s))
\le \bigl(q_{\mathrm S}'(0)\bigr)^5=C_1,
\qquad s\in[0,1),
\]
where the last equality follows from $C_1=a^5$ and,
by direct differentiation of $q_{\mathrm S}(s)=s\,p_{\mathrm S}(s^2)$,
$
q_{\mathrm S}'(0)=p_{\mathrm S}(0)=a.
$
Thus, $q_{\mathrm S,5}$ is strictly increasing and
$C_1$-Lipschitz on $[0,1]$.
Since $q_{\mathrm S,5}(1)=1$ and $C_1\ge1$ ($
p_{\mathrm S}(0)=1+\sum_{\ell=1}^{k}\frac{(2\ell)!}{4^\ell(\ell!)^2}>1$),
its extension $\tilde q_{\mathrm S,5}(s)$
  preserves both properties on $[0,\infty)$, where $\tilde q_{\mathrm S,5}(s)=q_{\mathrm S,5}(s)$
for $s\in[0,1]$ and $\tilde q_{\mathrm S,5}(s)=s$ for $s>1$.
The inverse $\tilde q_{\mathrm S,5}^{-1}$ is therefore
continuous and satisfies $\tilde q_{\mathrm S,5}^{-1}(0)=0$.
Recalling from \eqref{both:standard} that
$s\le q_{\mathrm S,5}(s)\le C_1s$ for $s\in[0,1]$, we have
\begin{equation}
\frac{u}{C_1}\le \tilde q_{\mathrm S,5}^{-1}(u)\le u,
\qquad u\ge0.
\label{eq:standard-proof-inverse-bounds}
\end{equation}
For $u\ge v\ge0$, we have
\[
u-v
=\tilde q_{\mathrm S,5}(\tilde q_{\mathrm S,5}^{-1}(u))
 -\tilde q_{\mathrm S,5}(\tilde q_{\mathrm S,5}^{-1}(v))
\le C_1\bigl(\tilde q_{\mathrm S,5}^{-1}(u)-\tilde q_{\mathrm S,5}^{-1}(v)\bigr),
\]
which implies
\begin{equation}
\tilde q_{\mathrm S,5}^{-1}(u)-\tilde q_{\mathrm S,5}^{-1}(v)\ge\frac{u-v}{C_1}.
\label{eq:standard-proof-inverse-increment}
\end{equation}

\medskip
\noindent\textbf{(i) Properties of the regularizer.}
By definition,
\[
\widetilde{\mathcal R}_t(W)
-\frac{\lambda\tau_t}{2C_1}\|W\|_F^2
=\frac{\tau_t}{\lambda}\sum_{i=1}^{r}
\int_0^{\lambda\sigma_i(W)}
\left(\tilde q_{\mathrm S,5}^{-1}(u)-\frac{u}{C_1}\right)\,du.
\]
\eqref{eq:standard-proof-inverse-bounds} and \eqref{eq:standard-proof-inverse-increment} imply that
the integrand is nonnegative and nondecreasing, so the integral is nondecreasing and convex.
By the convexity rule for singular-value functions, the
displayed matrix function is convex.
Hence $\widetilde{\mathcal R}_t$ is
$(\lambda\tau_t/C_1)$-strongly convex.

Since $\tilde q_{\mathrm S,5}^{-1}$ is continuous and
$\tilde q_{\mathrm S,5}^{-1}(0)=0$, the differentiation rule for
singular-value functions implies that
$\widetilde{\mathcal R}_t$ is continuously differentiable.
Let $u_i$ and $v_i$ denote the $i$-th columns of $U_W$ and
$V_W$, respectively.
Holding $\tau_t$ fixed, the differentiation formula for
singular-value functions
\citep[Theorem~3.1]{lewis1995convex} gives
\begin{align*}
\mathrm d\widetilde{\mathcal R}_t(W)
&=\tau_t\sum_{i=1}^{r}
\tilde q_{\mathrm S,5}^{-1}(\lambda\sigma_i(W))
\,u_i^\top(\mathrm dW)v_i\\
&=\Big\langle
\tau_t\sum_{i=1}^{r}
\tilde q_{\mathrm S,5}^{-1}(\lambda\sigma_i(W))
\,u_iv_i^\top,\,
\mathrm dW
\Big\rangle_F.
\end{align*}
This formula also holds at repeated or zero singular values.
Therefore,
\[
\nabla\widetilde{\mathcal R}_t(W)
=\tau_t U_W\operatorname{diag}\Bigl(
\tilde q_{\mathrm S,5}^{-1}(\lambda\sigma_1(W)),\ldots,
\tilde q_{\mathrm S,5}^{-1}(\lambda\sigma_r(W))
\Bigr)V_W^\top.
\]
Integrating \eqref{eq:standard-proof-inverse-bounds} gives
\[
\frac{\lambda^2\sigma_i(W)^2}{2C_1}
\le
\int_0^{\lambda\sigma_i(W)}
\tilde q_{\mathrm S,5}^{-1}(u)\,du
\le
\frac{\lambda^2\sigma_i(W)^2}{2}.
\]
Multiplying by $\tau_t/\lambda$ and summing over $i$ yields
\[
\frac{\lambda\tau_t}{2C_1}\|W\|_F^2
\le\widetilde{\mathcal R}_t(W)
\le\frac{\lambda\tau_t}{2}\|W\|_F^2.
\]
Finally, the gradient formula and
$\tilde q_{\mathrm S,5}^{-1}(u)\le u$ imply
\[
\|\nabla\widetilde{\mathcal R}_t(W)\|_F^2
=\tau_t^2\sum_{i=1}^{r}
\bigl[\tilde q_{\mathrm S,5}^{-1}(\lambda\sigma_i(W))\bigr]^2
\le\lambda^2\tau_t^2\|W\|_F^2.
\]
Taking square roots completes part (i).

\medskip
\noindent\textbf{(ii) Fixed-point equivalence and descent alignment.}
When $\beta=0$, we have $M_t=G_t$.
Recall that $s_{i,t}=\sigma_{i,t}/\tau_t\in(0,1]$, and
\[
G_t=U_t\operatorname{diag}
(\sigma_{1,t},\ldots,\sigma_{r,t})V_t^\top,
\qquad
D_t=U_t\operatorname{diag}\Bigl(
q_{\mathrm S,5}(s_{1,t}),\ldots,
q_{\mathrm S,5}(s_{r,t})
\Bigr)V_t^\top.
\]
Since
$\tilde q_{\mathrm S,5}^{-1}(q_{\mathrm S,5}(s))=s$
for $s\in[0,1]$, the gradient formula in (i) yields
\begin{align} \label{d:regu}
\nabla\widetilde{\mathcal R}_t(-D_t/\lambda)
=-\tau_t U_t\operatorname{diag}
(s_{1,t},\ldots,s_{r,t})V_t^\top
=-U_t\operatorname{diag}
(\sigma_{1,t},\ldots,\sigma_{r,t})V_t^\top
=-G_t.
\end{align}
Consequently,
\[
\widehat{\nabla}\widetilde\Phi_t(W_t)
=G_t+\nabla\widetilde{\mathcal R}_t(W_t)
=\nabla\widetilde{\mathcal R}_t(W_t)
-\nabla\widetilde{\mathcal R}_t(-D_t/\lambda).
\]
By the $(\lambda\tau_t/C_1)$-strong convexity of
$\widetilde{\mathcal R}_t$ a+nd
$H_t=D_t+\lambda W_t$, we have
\begin{align*}
\big\langle
\widehat{\nabla}\widetilde\Phi_t(W_t),H_t
\big\rangle_F
&=\lambda\big\langle
\nabla\widetilde{\mathcal R}_t(W_t)
-\nabla\widetilde{\mathcal R}_t(-D_t/\lambda),
W_t+D_t/\lambda
\big\rangle_F\\
&\ge\frac{\lambda^2\tau_t}{C_1}
\|W_t+D_t/\lambda\|_F^2
=\frac{\tau_t}{C_1}\|H_t\|_F^2.
\end{align*}
If $H_t=0$, then $W_t=-D_t/\lambda$, and hence
\[
\widehat{\nabla}\widetilde\Phi_t(W_t)
=\nabla\widetilde{\mathcal R}_t(W_t)
-\nabla\widetilde{\mathcal R}_t(-D_t/\lambda)
=0.
\]
Conversely, if
$\widehat{\nabla}\widetilde\Phi_t(W_t)=0$,
then
\[
0
=\big\langle
\widehat{\nabla}\widetilde\Phi_t(W_t),H_t
\big\rangle_F
\ge\frac{\tau_t}{C_1}\|H_t\|_F^2.
\]
Since $\tau_t/C_1>0$, this implies $H_t=0$.
Thus,
\[
H_t=0
\quad\Longleftrightarrow\quad
\widehat{\nabla}\widetilde\Phi_t(W_t)=0.   
\]

\medskip
\noindent\textbf{(iii) Exact generalized Frank--Wolfe equivalence.}
Fix an iteration $t$ and a realization of $G_t$.
Throughout this step, $\tau_t$ is held fixed.
Since $\beta=0$, \eqref{d:regu} still holds.
Then, by the $(\lambda\tau_t/C_1)$-strong convexity of
$\widetilde{\mathcal R}_t$, for every
$W\in\mathbb R^{m\times n}$,
 \begin{align*}
\widetilde{\mathcal R}_t(W)
\ge 
\widetilde{\mathcal R}_t\Big(\frac{-D_t}{\lambda}\Big)
-\Big\langle G_t,W+\frac{D_t}{\lambda}\Big\rangle_F +\frac{\lambda\tau_t}{2C_1}
\Big\|W+\frac{D_t}{\lambda}\Big\|_F^2.
\end{align*}
Adding $\langle G_t,W\rangle_F$ to both sides gives
\begin{align*}
\langle G_t,W\rangle_F+\widetilde{\mathcal R}_t(W)
\ge{}&
\Big\langle G_t,-\frac{D_t}{\lambda}\Big\rangle_F
+\widetilde{\mathcal R}_t\Big(\frac{-D_t}{\lambda}\Big)
+\frac{\lambda\tau_t}{2C_1}
\Big\|W+\frac{D_t}{\lambda}\Big\|_F^2.
\end{align*}
Since $\lambda\tau_t/C_1>0$, the last term is strictly
positive whenever $W\ne-D_t/\lambda$.
Thus $-D_t/\lambda$ is the unique minimizer, proving
\eqref{eq:standard-gfw-oracle}.

Finally, substituting $\eta=\gamma/\lambda$ into the
Muon update gives
\[
W_{t+1}
=W_t-\eta(D_t+\lambda W_t)=W_t-\frac{\gamma}{\lambda}(D_t+\lambda W_t)
=(1-\gamma)W_t
+\gamma\big(-D_t/\lambda\big).  \qedhere
\]
\end{proof}

\end{document}